\documentclass{article}
\usepackage[nonatbib, preprint]{neurips_2021}
\usepackage[utf8]{inputenc} 
\usepackage[T1]{fontenc}    

\usepackage{hyperref}       

\hypersetup{
    colorlinks=true,
    citecolor=yale,
    linkcolor=yale,
    filecolor=magenta,      
    urlcolor=Black,
}

\usepackage{url}            
\usepackage{booktabs}       
\usepackage{amsfonts}       
\usepackage{nicefrac}       
\usepackage{microtype}      
\usepackage[svgnames]{xcolor}
\definecolor{yale}{RGB}{14,77,146}

\usepackage{mathtools}
\usepackage{amsmath}
\usepackage{amssymb}
\usepackage{amsthm}
\usepackage{bbm}
\usepackage{dsfont}
\usepackage{bm}
\usepackage{graphicx}
\usepackage{subfigure}
\usepackage{theoremref}

\usepackage{tikz}  
\usetikzlibrary{automata, positioning}
\usepackage[ruled,linesnumbered]{algorithm2e}

\title{On the complexity of All $\varepsilon$-Best Arms Identification}

\author{%
  Aymen Al Marjani
  \\ UMPA, ENS Lyon\\ 
  Lyon, France \\ \texttt{aymen.al\_marjani@ens-lyon.fr} \\
     \And
  Tomas Kocak \\
   University of Potsdam \\
   
   \texttt{tomas.kocak@gmail.com} \\
   \AND
   Aurélien Garivier \\
   UMPA, CNRS, INRIA, ENS Lyon \\
   Lyon, France \\
   \texttt{aurelien.garivier@ens-lyon.fr} \\

}

\newtheorem{theorem}{Theorem}

\newtheorem{lemma}[theorem]{Lemma}

\newtheorem{proposition}[theorem]{Proposition}

\theoremstyle{definition}
\newtheorem{definition}{Definition}

\newcommand{\Alt}[1]{\operatorname{Alt}(#1)}
\newcommand{\KL}{\operatorname{KL}}

\DeclareMathOperator*{\argmin}{arg\,min}

\DeclarePairedDelimiter\floor{\lfloor}{\rfloor}

\newcommand{\dsR}{\mathds{R}}

\newcommand{\E}{\mathbb{E}}
\renewcommand{\P}{\mathbb{P}}

\newcommand{\cA}{\mathcal{A}}

\newcommand{\cD}{\mathcal{D}}

\newcommand{\cN}{\mathcal{N}}
\newcommand{\cO}{\mathcal{O}}

\newcommand{\CT}{T_{\varepsilon}^*(\bmu)}

\newcommand{\bd}{\bm{d}}

\newcommand{\bmu}{{\boldsymbol \mu}}
\newcommand{\bnu}{{\boldsymbol \nu}}

\newcommand{\bomega}{{\boldsymbol \omega}}

\newcommand{\blambda}{\boldsymbol \lambda}
\newcommand{\blambdakl}{\blambda^{k,\ell}_{\varepsilon}(\bomega)}
\newcommand{\blambdaklp}{\blambda'^{k,\ell}_{\varepsilon}(\bomega)}
\newcommand{\blambdastar}{\blambda^*_{\varepsilon, \bmu}(\bomega)}

\newcommand{\suma}{\sum_{a\in[K]}}
\newcommand{\suml}{\sum_{a=1}^\ell}
\newcommand{\simplex}[1]{\Delta_#1}
\newcommand{\G}[1]{G_{\varepsilon}(#1)}
\newcommand{\B}[1]{B_{\varepsilon}(#1)}

\newcommand{\transpose}{^\mathsf{\scriptscriptstyle T}}
\newcommand{\norm}[1]{\left\lVert#1\right\rVert}
\newcommand{\TV}[1]{\mathrm{TV}(#1)}
\newcommand{\kl}{\operatorname{kl}}

\newcommand{\Dset}{\cD_{\varepsilon,\,\bmu}}

\newcommand{\avg}{\overline{\bmu}_{\varepsilon}^{k,\ell}(\bomega)}
\newcommand{\avgp}{\overline{\bmu'}_{\varepsilon}^{k,\ell}(\bomega)}
\renewcommand{\epsilon}{\varepsilon}
\newcommand{\ie}{i.e.\xspace}

\begin{document}

\maketitle
%


\begin{abstract}
  We consider the question introduced by \cite{Mason2020} of identifying all the $\varepsilon$-optimal arms in a finite stochastic multi-armed bandit with Gaussian rewards. We give two lower bounds on the sample complexity of any algorithm solving the problem with a confidence at least $1-\delta$. The first, unimprovable in the asymptotic regime, motivates the design of a Track-and-Stop strategy whose average sample complexity is asymptotically optimal when the risk $\delta$ goes to zero. Notably, we provide an efficient numerical method to solve the convex max-min program that appears in the lower bound. Our method is based on a complete characterization of the alternative bandit instances that the optimal sampling strategy needs to rule out, thus making our bound tighter than the one provided by \cite{Mason2020}. The second lower bound deals with the regime of high and moderate values of the risk $\delta$, and characterizes the behavior of any algorithm in the initial phase. It emphasizes the linear dependency of the sample complexity in the number of arms. Finally, we report on numerical simulations demonstrating our algorithm's advantage over state-of-the-art methods, even for moderate risks.
  
\end{abstract}

\section{Introduction}

The problem of finding all the $\varepsilon$-good arms was recently introduced by \cite{Mason2020}. For a finite family of distributions $(\bnu_a)_{a \in [K]}$ with vector of mean rewards $\bmu = (\mu_a)_{a \in [K]}$, the goal is to return 
$\G{\bmu} \triangleq \{ a\in[K]:\mu_a \geq \max_{i} \mu_i - \varepsilon \}$ in the additive case and $\G{\bmu} \triangleq \{ a\in[K]:\mu_a \geq (1 - \varepsilon) \max_{i} \mu_i \}$ in the multiplicative case. This problem is closely related to two other pure-exploration problems in the multi-armed bandit literature, namely the TOP$-k$ arms selection and the THRESHOLD bandits. The former aims to find the $k$ arms with the highest means, while the latter seeks to identify all arms with means larger than a given threshold $s$. As argued by \cite{Mason2020}, finding all the $\varepsilon$-good arms is a more robust objective than the TOP-K and THRESHOLD problems, which require some prior knowledge of the distributions in order to return a relevant set of solutions. Take for example drug discovery applications, where the goal is to perform an initial selection of potential drugs through \textit{in vitro} essays before conducting more expensive clinical trials: setting the number of arms $k$ too high or the threshold $s$ too low may result into poorly performing solutions. Conversely, if we set $k$ to a small number or the threshold $s$ too high we might miss promising drugs that will prove to be more efficient under careful examination. The All-$\varepsilon$ objective circumvents this issues by requiring to return all drugs whose efficiency lies within a certain range from the best.
In this paper, we want to identify $\G{\bmu}$ in a PAC learning framework with fixed confidence: for a risk level $\delta$, the algorithm samples arms $a \in [K]$ in a sequential manner to gather information about the distribution means $(\mu_a)_{a\in [K]}$ and returns an estimate $\widehat{G}_\varepsilon$ such that $\P_{\bmu}(\widehat{G}_\varepsilon \neq \G{\bmu}) \leq \delta$. Such an algorithm is called $\delta$-PAC and its performance is measured by the expected number of samples $\E[\tau_\delta]$, also called the \textit{sample complexity}, needed to return a good answer with high probability. \cite{Mason2020} provided two lower bounds on the sample complexity: fhe first bound is based on a classical change-of-measure argument and exhibits the behavior of sample complexity in the low confidence regime ($\delta\to0$). The second bound resorts to the Simulator technique~\cite{simchowitz17a} combined with an algorithmic reduction to Best Arm Identification and shows the dependency of the sample complexity on the number of arms $K$ for moderate values of $\delta$. They also proposed \textrm{FAREAST}, an algorithm matching the first lower bound, up to some numerical constants and log factors, in the asymptotic regime $\delta\to0$. Our contributions can be summarized as follows: 
\begin{itemize}
    \item Usual lower bounds on the sample complexity write as $f(\nu)\log(1/\delta) + g(\nu)$ for an instance $\nu$. We derive a tight bound in terms of the first-order term which writes as $\CT\log(1/\delta)$, where the characteristic time $\CT$ is the value of a concave max-min optimization program. Our bound is tight in the sense that any lower bound of the form $f(\nu)\log(1/\delta)$ that holds for all $\delta \in (0,1)$ is such that $f(\nu) \leq \CT$. To do so, we investigate all the possible alternative instances $\blambda$ that one can obtain from the original problem $\bmu$ by a change-of-measure, including (but not only) the ones that were considered by \cite{Mason2020}. 
    \vspace{0.05in}
    \item  We derive a second lower bound that writes as $g(\nu)$ in Theorem \ref{thm:simulator_LB}. $g(\nu)$ shows an additional linear dependency on the number of arms which is negligible when $\delta \to 0$ but can be dominant for moderate values of the risk. This result generalizes Theorem 4.1 in \cite{Mason2020}, since it also includes cases where there can be several arms with means close to the top arm. The proof of this result relies on a personal rewriting  of the Simulator method of \cite{simchowitz17a} which was proposed for the Best Arm Identification and TOP-k problems. As we explain in Section \ref{sec:simulator_lb}, our proof can be adapted to derive lower bounds for other pure exploration problems, {\it without resorting to algorithmic reduction of these problems to Best Arm Identification}. Therefore, we believe that the proof itself constitutes a significant contribution.
    \vspace{0.05in}
    \item We present two efficient methods to solve the minimization sub-problem (resp. the entire max-min program) that defines the characteristic time. These methods are used respectively in the stopping and sampling rule of our Track-and-Stop algorithm, whose sample complexity matches the lower bound when $\delta$ tends to 0. 
    \vspace{0.05in}
    \item Finally, to corroborate our asymptotic results, we conduct numerical experiments for a wide range of the confidence parameters and number of arms. Empirical evaluation shows that Track-and-Stop is optimal either for a small number of arms $K$ or when $\delta$ goes to 0, and excellent in practice for much larger values of $K$ and $\delta$. We believe these are significant improvements in performance to be of interest for ML practitioners seeking solutions for this kind of problem.
\end{itemize} 

Similar to previous works, we restrict our attention to bandits with rewards coming from a Gaussian distribution with variance one. Even though this assumption is not mandatory, it considerably simplifies the presentation of the results\footnote{For $\sigma^2$-subgaussian distributions, we only need to  multiply our bounds by $\sigma^2$. For bandits coming from another single-parameter exponential family, we lose the closed-form expression of the best response oracle that we have in the Gaussian case, but one can use binary search to solve the best response problem.}. Section \ref{sec:LB} is devoted to our lower bounds on the sample complexity of identifying the set of $\varepsilon$-good arms and the pseudo-code of our algorithm, along with the theoretical guarantees on its sample complexity. In Sections \ref{sec:min} and \ref{sec:max_min}, we present our method for solving the optimization program that defines the characteristic time, which is at the heart of the sampling and stopping rules of our algorithm.

\section{Lower bounds and asymptotically matching algorithm}\label{sec:LB}
We start by proving a lower bound on the sample complexity of any $\delta$-correct algorithm. This lower bound will later motivate the design of our algorithm.
\subsection{First lower bound}
Let $\simplex{K}$ denote the $K$-dimensional simplex and $\kl(p, q)$ be the KL-divergence between two Bernoulli distributions with parameters $p$ and $q$. Finally, define the set of \emph{alternative} bandit problems $\Alt\bmu = \{\blambda\in \dsR^K: \G{\bmu} \neq \G{\blambda}\}$.
Using change-of-measure arguments introduced by \cite{lai1985} 
, we derive the following lower bound on the sample complexity in our special setting.

\begin{proposition}
  \label{proposition:lb}
  For any $\delta$-correct strategy and any bandit instance $\bmu$, the expected stopping time $\tau_\delta$ can be lower-bounded as
  \[
    \E[\tau_\delta] \ge \CT\log(1/2.4\delta)
  \]
  where \begin{align}
     & \CT^{-1} \triangleq \sup_{\bomega\in\simplex{K}} T_{\varepsilon}(\bmu,\bomega)^{-1} \qquad \hbox{and} \label{eq:CT}                                      \\
     & T_{\varepsilon}(\bmu,\bomega)^{-1} \triangleq \inf_{\blambda\in\Alt\bmu}\suma\omega_a\frac{(\mu_a - \lambda_a)^2}{2} \;.\label{eq:CT_inf}
  \end{align}
\end{proposition}
The characteristic time $\CT$ above is an instance-specific quantity that determines the difficulty of our problem. The optimization problem in the definition of $\CT$ can be seen as a two-player game between an algorithm which samples each arm $a$ proportionally to $\omega_a$ and an adversary who chooses an alternative instance $\blambda$ that is difficult to distinguish from $\bmu$ under the algorithm's sampling scheme. This suggests that an optimal strategy should play the optimal allocation $\bomega^*$ that maximizes the optimization problem (\ref{eq:CT}) and, as a consequence, rules out all alternative instances as fast as possible. This motivates our algorithm, presented in Section \ref{sec:algo}.

\subsection{Algorithm}\label{sec:algo}
We propose a simple Track-and-Stop strategy similar to the one proposed by \cite{GK16} for the problem of Best-Arm Identification. It starts by sampling once from every arm $a\in[K]$ and constructs an initial estimate $\widehat{\bmu}_K$ of the vector of mean rewards $\bmu$. After this burn-in phase, the algorithm enters a loop where at every iteration it samples according to the estimated optimal sampling rule (\ref{eq:sampling_rule}) and updates its estimate $\widehat{\bmu}_t$ of the arms' expectations. Finally, the algorithm checks if the stopping rule (\ref{eq:stopping_rule}) is satisfied, in which case it stops and returns the set of empirically $\varepsilon$-good arms.
\paragraph{\bfseries Sampling rule:}
our sampling rule performs so-called C-tracking: first, we compute $\widetilde{\bomega}(\widehat{\bmu}_t)$, an allocation vector which is $\frac{1}{\sqrt{t}}$-optimal in the lower-problem (\ref{eq:CT}) for the instance $\widehat{\bmu}_t$. Then we project $\widetilde{\bomega}(\widehat{\bmu}_t)$ on the set $\Delta_K^{\eta_t} = \Delta_K \cap [\eta_t,1]^K$. Given the projected vector $\widetilde{\bomega}^{\eta_t}(\widehat{\bmu}_t)$, the next arm to sample from is defined by:
\begin{equation}
    a_{t+1} = \argmin\limits_{a} N_a(t) - \sum\limits_{s=1}^t \widetilde{\bomega}_a^{ \eta_t}(\widehat{\bmu}_s)
    \label{eq:sampling_rule}
\end{equation}
where $N_a(t)$ is the number of times arm $a$ has been pulled up to time $t$. In other words, we sample the arm whose number of visits is farther behind its corresponding sum of empirical optimal allocations. In the long run, as our estimate $\widehat{\bmu}_t$ tends to the true value $\bmu$, the sampling frequency $N_a(t)/t$ of every arm $a$ will converge to the oracle optimal allocation $\bomega^*_a(\bmu)$. The projection on $\Delta_K^{\eta_t}$  ensures exploration at minimal rate $\eta_t = \frac{1}{2\sqrt{(K^2 + t)}}$ so that no arm is left-behind because of bad initial estimates.

\paragraph{\bfseries Stopping rule:}
To be sample-efficient, the algorithm should should stop as soon as the collected samples are sufficiently informative to declare that $G_\varepsilon(\widehat{\bmu}_t) = G_\varepsilon(\bmu)$ with probability larger than $1-\delta$. For this purpose we use the Generalized Likelihood Ratio (GLR) test \cite{Chernoff59}. We define the $Z$-statistic:
\begin{equation*}
    Z(t) = t \times T_\varepsilon\bigg(\widehat{\bmu}_t, \frac{N(t)}{t} \bigg)^{-1}
\end{equation*}
where $N(t) = \big(N_a(t)\big)_{a\in [K]}$. As shown in \cite{GK16,GK19PAC}, the Z-statistic is equal to the ratio of the likelihood of observations under the most likely model where $G_\varepsilon(\widehat{\bmu}_t)$ is the correct answer, \ie $\widehat{\bmu}_t$, to the likelihood of observations under the most likely model where $G_\varepsilon(\widehat{\bmu}_t)$ is not the set of $\varepsilon$-good arms. The algorithm rejects the hypothesis $G_\varepsilon(\widehat{\bmu}_t) \neq G_\varepsilon(\bmu)$ and stops as soon as this ratio of likelihoods becomes larger than a certain threshold $\beta(\delta,t)$, properly tuned to ensure that the algorithm is $\delta$-PAC. The stopping rule is defined as:
\begin{equation}
    \tau_\delta = \inf\big\{ t \in \mathbb{N}\ :\ Z(t) > \beta(t,\delta) \big\}
    \label{eq:stopping_rule}
\end{equation}

One can find many suitable thresholds from the bandit literature \cite{Garivier2013}, \cite{magureanu2014}, \cite{Kaufmann18MixtureMR}, all of which are of the order $\beta(\delta,t) \approx \log(1/\delta) + \frac{K}{2} \log(\log(t/\delta)) $ is enough to ensure that $\P\big(G_\varepsilon(\widehat{\bmu}_{\tau_\delta}) \neq G_\varepsilon(\bmu)\big) \leq \delta$, \ie that the algorithm is $\delta$-correct.

\begin{algorithm}\label{main_algo} 
    \caption{Track and Stop}
        \KwIn{Confidence level $\delta$, accuracy parameter $\varepsilon$.}
        Pull each arm once and observe rewards $(r_a)_{a \in [K]}$.\\
        Set initial estimate $\widehat{\bmu}_K = (r_1,\ldots,r_K)^{T}$.\\
        Set $t \leftarrow K$ and $N_a(t) \leftarrow 1$ for all arms $a$.\\
        \While{Stopping condition (\ref{eq:stopping_rule}) is not satisfied }
        {
        Compute $\Tilde{\bomega}(\widehat{\bmu}_t)$, a $\frac{1}{\sqrt{t}}$-optimal vector for (\ref{eq:CT}) using mirror-ascent.\\
        Pull next arm $a_{t+1}$ given by (\ref{eq:sampling_rule}) and observe reward $r_t$.\\
        Update $\widehat{\bmu}_t$ according to $r_t$.\\
        Set $ t \leftarrow t+1$ and update $\big(N_a(t)\big)_{a \in [K]}$.
        }
        \KwOut{Empirical set of $\varepsilon$-good arms: $G_\varepsilon(\widehat{\bmu}_{\tau_\delta})$}
\end{algorithm}
Now we state our sample complexity result which we adapted from Theorem 14 in \cite{GK16}. Notably, while their Track-and-Stop strategy relies on tracking the exact optimal weights to prove that the expected stopping time matches the lower bound when $\delta$ tends to zero, our proof shows that it is enough to track some slightly sub-optimal weights with a decreasing gap in the order of $\frac{1}{\sqrt{t}}$ to enjoy the same sample complexity guarantees.
\begin{theorem}
    For all $\delta \in (0,1)$, Track-and-Stop terminates almost-surely and its stopping time $\tau_\delta$ satisfies:
    \begin{equation*}
        \limsup\limits_{\delta\to 0} \frac{\E[\tau_\delta]}{\log(1/\delta)} \leq \CT.\
    \end{equation*}
\label{thm:upper_bound}
\end{theorem}
{\bfseries Remark 1.}
    Suppose that the arms are ordered  decreasingly $\mu_1\ge \mu_2 \ge \dots\ge \mu_K$. \cite{Mason2020} define the upper margin $\alpha_\varepsilon = \min\limits_{k\in\G\bmu} \mu_k -(\mu_1 -\varepsilon)$ and provide a lower bound of the form $f(\nu)\log(1/\delta)$ where:
    \begin{align*}
        f(\nu) \triangleq 2 \sum\limits_{a=1}^{K} \max\bigg(\frac{1}{(\mu_1 - \varepsilon -\mu_i)^2},\frac{1}{(\mu_1 + \alpha_\varepsilon -\mu_a)^2}\bigg).
    \end{align*}
    It can be seen directly (or deduced from Theorem \ref{thm:upper_bound}) that $f(\nu) \leq \CT$. In a second step, they proposed $\mathrm{FAREAST}$, an algorithm  whose sample complexity in the asymptotic regime $\delta \to 0$ matches their bound up to some universal constant $c$ that does not depend on the instance $\nu$. From Proposition \ref{proposition:lb}, we deduce that $\CT \leq c f(\nu)$, which can be seen directly from the particular changes of measure considered in that paper. The sample complexity of our algorithm improves upon previous work by multiplicative constants that can possibly be large, as illustrated in Section~\ref{sec:experiments}.

\subsection{Lower bound for moderate confidence regime}\label{sec:simulator_lb}
The lower bound in Proposition \ref{proposition:lb} and the upper bound in Theorem \ref{thm:upper_bound} show that in the asymptotic regime $\delta\to 0$ the optimal sample complexity scales as $\CT\log(1/\delta)$. However, one may wonder whether this bound catches all important aspects of the complexity, especially for large or moderate values of the risk $\delta$. Towards answering this question, we present the following lower bound which shows that there is an additional cost, linear in the number of arms, that any $\delta$-PAC algorithm must pay in order to learn the set of All-$\epsilon$ good arms. Before stating our result, let us introduce some notation. We denote by $\mathbf{S}_K$ the group of permutations over $[K]$. For a bandit instance $\nu = (\nu_1, \ldots, \nu_K)$ we define the {\it permuted instance} $\pi(\nu) = (\nu_{\pi(1)}, \ldots, \nu_{\pi(K)})$. $\mathbf{S}_K(\nu) = \{\pi(\nu),\ \pi \in \mathbf{S}_K\}$ refers to the set of all permuted instances of $\nu$. Finally, we will write $\pi \sim \mathbf{S}_K$ to indicate that a permutation is drawn uniformly at random from $\mathbf{S}_K$. These results are much inspired from~\cite{Mason2020}, but come with quite different proofs that we hope can be useful to the community.
\begin{theorem}
Fix $\delta \leq 1/10$ and $\epsilon > 0$. Consider an instance $\nu$ such that there exists at least one bad arm: $G_\epsilon(\bmu)\neq [K]$. Without loss of generality, suppose the arms are ordered  decreasingly $\mu_1\ge \mu_2 \ge \dots\ge \mu_K$ and define the lower margin $\beta_\epsilon = \min\limits_{k\notin\G\bmu} \mu_1 -\varepsilon - \mu_k$. Then any $\delta$-PAC algorithm has an average sample complexity over all permuted instances satisfying 
\begin{align*}
    \E_{\pi \sim \mathbf{S}_K}\E_{\pi(\nu)}[\tau_\delta] \geq \frac{1}{12|G_{\beta_\epsilon}(\bmu)|^3}\sum_{b=1}^K \frac{1}{(\mu_1 - \mu_b + \beta_\epsilon)^2},
\end{align*}
\label{thm:simulator_LB}
\end{theorem}
The proof of the lower bound can be found in Appendix \ref{sec:proof_simulator}. In the special case where $|G_{2\beta_\epsilon}|=1$, then $|G_{\beta_\epsilon}|=1$ also (since $\{1 \}\subset G_{\beta_\epsilon} \subset G_{2\beta_\epsilon}$) and we recover the bound in Theorem 4.1 of \cite{Mason2020}. The lower bound above informs us that we must pay a linear cost in $K$, {\it even when there are several arms close to the top one}, provided that their cardinal does not scale with the total number of arms, i.e. $|G_{\beta_\epsilon}| = \cO(1)$.

{\noindent \bfseries The bound of Thm \ref{thm:simulator_LB} can be arbitrarily large compared to $\CT\log(1/\delta)$.} Fix  $\delta = 0.1$ and let $\epsilon, \beta > 0$ with $\beta \ll \epsilon$ and consider the instance such that $\mu_1 = \beta, \mu_K = -\epsilon$ and $\mu_a = -\beta$ for $a\in [|2,K-1|]$. Then we show in Appendix \ref{sec:proof_simulator} that $\CT\log(1/\delta) = \cO(1/\beta^2 + K/\epsilon^2)$. In contrast the lower bound above scales as $\Omega(K/\beta^2)$. Since $\beta \ll \epsilon$, the second bound exhibits a better scaling w.r.t the number of arms.

{\bfseries The intuition behind this result} comes from the following observations: first, note that arms in $G_{\beta_\epsilon}(\bmu)$ must be sampled at least $\Omega(1/\beta_\epsilon^2)$ times, because otherwise we might underestimate their means and misclassify the arms in $\argmin_{k\notin\G\bmu} \mu_1 -\varepsilon - \mu_k$ as good arms. Second, in the initial phase the algorithm does not know which arms belong to $G_{\beta_\epsilon}(\bmu)$ and we need at least $\Omega(1/(\mu_1-\mu_b)^2)$ samples to distinguish any arm $b$ from arms in $G_{\beta_\epsilon}(\bmu)$. Together, these observations tell us that we must pay a cost of $\Omega(\min(1/\beta_\epsilon^2, 1/(\mu_1-\mu_b)^2))$ samples to either declare that $b$ is not in $G_{\beta_\epsilon}(\bmu)$ or learn its mean up to $\cO(\beta_\epsilon)$ precision. More generally, consider a pure exploration problem with a unique answer, where some particular arm $i^\star$ \footnote{or a subset of arms, as in our case.} needs to be estimated up to some precision $\eta > 0$ in order to return the correct answer. In this case, one can adapt our proof, {\it without using any algorithmic reduction to Best Arm Identification}, to show that every arm $a$ must be played at least $\Omega(1/(|\mu_{i^\star} - \mu_a|+ \eta)^2)$ times. For example, consider the problem of testing whether the minimum mean of a multi-armed bandit is above or below some threshold $\gamma$. Let $\nu$ be an instance such that $\{a\in [K]:\ \mu_a < \gamma \} \neq \emptyset$. Define $\eta = \min_{a : \mu_a < \gamma} \gamma - \mu_{a} > 0$ and let $i^\star$ to be the arm that achieves this minimum. Then our proof can be adapted in a straightforward fashion to prove that any $\delta$-PAC has a sample complexity of at least $\Omega\big(\sum_{a=1}^K \frac{1}{(\mu_a - \mu_{i^\star} + \eta)^2}\big)$.\footnote{
The phenomenon discussed above is essentially already discussed in~\cite{Mason2020}, a very rich study of the problem. However, we do not fully understand the proof of Theorem 4.1.
Define a sub-instance to be a bandit $\widetilde{\nu}$ with fewer arms $m \leq K$ such that $\{\widetilde{\nu}_1,\ldots, \widetilde{\nu}_{m}\} \subset \{\nu_1, \ldots, \nu_K\}$.  Lemma D.5 actually shows that there exists some sub-instance of $\nu$ on which the algorithm must pay $\Omega(\sum_{b=2}^{m} 1/(\mu_1-\mu_b)^2)$ samples. But this does not imply that such cost must be paid for the instance of interest $\nu$ instead of some sub-instance with very few arms.
}

Note that Algorithm \ref{main_algo} requires to solve the best response problem, \ie the minimization problem in (\ref{eq:CT_inf}), in order to be able to compute the $Z$-statistic of the stopping rule, and also to solve the entire lower bound problem in (\ref{eq:CT}) to compute the optimal weights for the sampling rule. The rest of the paper is dedicated to presenting the tools necessary to solve these two problems.

\section{Solving the min problem: Best response oracle}\label{sec:min}
For a given vector $\bomega$, we want to compute the best response
\begin{equation}
  \blambdastar \triangleq \argmin_{\blambda\in\Alt\bmu} \suma \omega_a\frac{(\mu_a - \lambda_a)^2}{2}.
  \label{eq:objective}
\end{equation}
For the simplicity of the presentation, we assume that the arms are ordered  decreasingly $\mu_1\ge \mu_2 \ge \dots\ge \mu_K$ and start by presenting the additive case (\ie $\G{\bmu} \triangleq \{ a\in[K]:\mu_a \geq \max\limits_{i} \mu_i - \varepsilon \}$). The multiplicative case can be treated in the same fashion and is deferred to appendix \ref{sec:multiplicative}. Finally, we denote by $\B{\bmu} \triangleq [K]\setminus \G{\bmu}$ the set of bad arms.


Since an alternative problem $\blambda\in\Alt\bmu$ must have a different set of $\varepsilon$-optimal arms than the original problem $\bmu$, we can obtain it from $\bmu$ by changing the expected reward of some arms. We have two options to create an alternative problem $\blambda$:
\begin{itemize}
  \item \textbf{Making one of the $\varepsilon$-optimal arms bad}. We can achieve it by decreasing the expectation of some $\varepsilon$-optimal arm $k$ while increasing the expectation of some other arm $\ell$ to the point where $k$ is no more $\varepsilon$-optimal. This is illustrated in Figure \ref{img:FN}.
  \item \textbf{Making one of the $\varepsilon$-sub-optimal arms good.} We can achieve it by increasing the expectation of some sub-optimal arm $k$ while decreasing the expectations of the arms with the largest means -as many as it takes- to the point where $k$ becomes $\varepsilon$-optimal. This is illustrated in Figure \ref{img:FN}.
\end{itemize}

\begin{figure}[h]
    \centering
    \includegraphics[width=0.49\textwidth]{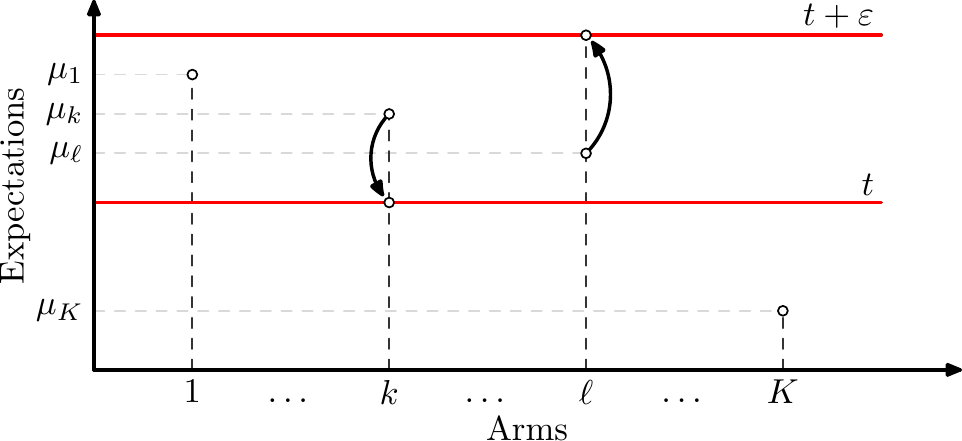}
    \includegraphics[width=0.5\textwidth]{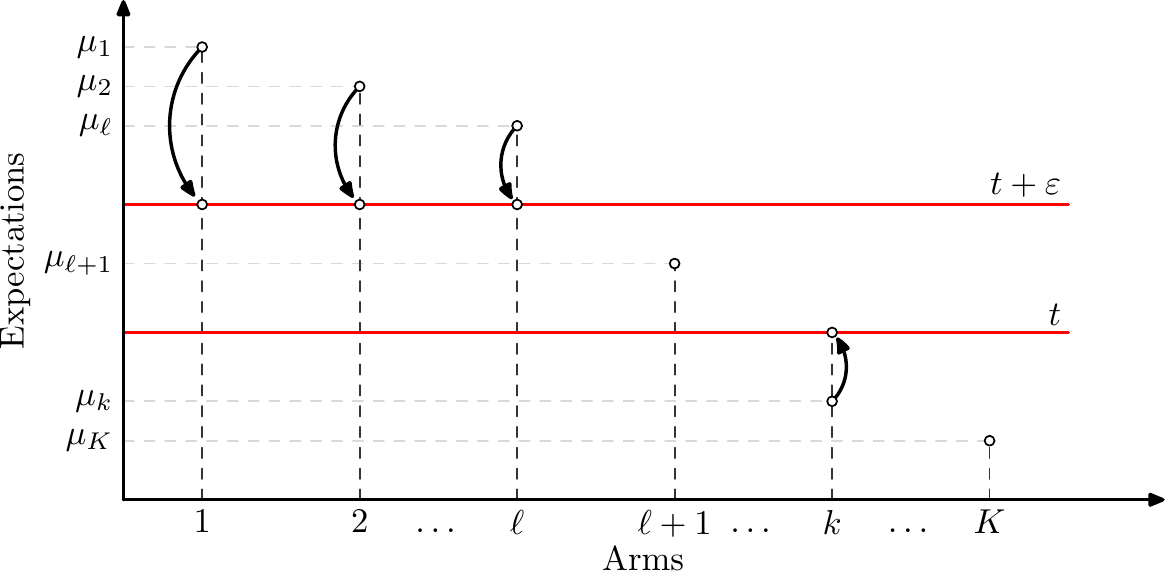}
    \caption{Left: Making One of the $\varepsilon$-Optimal Arms Bad. Right: Making One of the $\varepsilon$-Sub-Optimal Arms Good.}
    \label{img:FN}
  \end{figure}




In the following, we solve both cases separately.

\paragraph{\bf Case 1: Making one of the $\varepsilon$-optimal arms bad.}

Let $k \in \G\bmu$ be one of the $\varepsilon$-optimal arms. In order to make arm $k$ sub-optimal, we need to set the expectation of arm $k$ to some value $ \lambda_k = t$ and the maximum expectation over all arms to $\max\limits_{a} \lambda_a = t+\varepsilon$. Note that the index of the arm $\ell$ with maximum expectation can be chosen in $\G\bmu$. Indeed, if we choose some arm from $\B\mu$ to become the arm with maximum expectation in $\lambda$ then we would make an $\varepsilon$-suboptimal arm good which is covered in the other case below. The expectations of all the other arms should stay the same as in the instance $\bmu$, since changing their values would only increase the value of the objective. Now given indices $k$ and $\ell$, computing the optimal value of $t$ is rather straightforward since the objective function simplifies to
\[
  \omega_k\frac{(\mu_k-t)^2}{2} + \omega_\ell\frac{(\mu_\ell-t-\varepsilon)^2}{2}
\]
for which the optimal value of $t$ is:
\[
  t = \avg \triangleq \frac{\omega_k\mu_k + \omega_\ell(\mu_\ell - \varepsilon) }{\omega_k+\omega_\ell}.
\]
and the corresponding alternative bandit is:
\[
  \blambdakl \triangleq (\mu_1,\dots,\underbrace{\avg}_{\mathrm{index\ }k},\dots,\underbrace{\avg+\varepsilon}_{\mathrm{index\ }\ell},\dots,\mu_K )\transpose\!.
\]
The last step is taking the pair of indices $(k,\ell) \in \G\bmu\times(\G\bmu\setminus\{k\})$ with the minimal value in the objective (\ref{eq:CT_inf}).

\paragraph{\bf Case 2: Making one of the sub-optimal arms good.}

Let $k\in \B\bmu$ be a sub-optimal arm, if such arm exists, and denote by $t$ the value of its expectation in $\blambda$. In order to make this arm $\varepsilon$-optimal, we need to decrease the expectations of all the arms that are above the threshold $t+\varepsilon$. We pay a cost of $\frac{1}{2} \omega_k (t-\mu_k)^2$ for moving arm $k$ and of $\frac{1}{2} \omega_i (t+\varepsilon-\mu_i)^2$ for every arm $i$ such that $\mu_i > t+\varepsilon$.
Consider the functions:
\[
  f_k(t) = \frac{1}{2}\omega_k(t - \mu_k)^2
\]
and for $i\in [K]\setminus\{k\}$
\[
  f_i(t) =
  \begin{cases}
    \begin{array}{ll}
      \frac{1}{2}\omega_i (t+\varepsilon - \mu_i)^2 & \mathrm{for}\ t < \mu_i -\varepsilon,   \\
      0                                             & \mathrm{for}\ t\geq \mu_i -\varepsilon.
    \end{array}
  \end{cases}
\]
Each of these functions is convex. Therefore the function $f(t) = \sum\limits_{i=1}^K f_i(t)$ is convex and has a unique minimizer $t^*$. One can easily check that $f'(\mu_k) \leq 0$ and $f'(\mu_1-\varepsilon) \geq 0$, implying that $\mu_k - \varepsilon < \mu_k \leq t^* \leq \mu_1 - \varepsilon$. Therefore:
\[
  \ell = \min \{i \geq 1 \ :\ t^* > \mu_i - \varepsilon \} -1
\]
is well defined and satisfies $\ell \in [|1,k-1|]$. Note that by definition $\mu_{\ell+1} - \varepsilon < t^*$  and $t^* \leq \mu_a - \varepsilon$ for all $a\leq \ell$, hence:
\[
  0 = f'(t^*) = \omega_k (t^* - \mu_k) + \suml \omega_a (t^* + \varepsilon - \mu_a).
\]
Implying that\footnote{$\avg$ has a different definition depending on $k$ being a good or a bad arm.}:
\[
  t^* = \avg \triangleq \frac{\omega_k\mu_k + \suml \omega_a(\mu_a-\varepsilon)}{\omega_k + \suml \omega_a}
\]
and the alternative bandit in this case writes as:
\[
  \blambdakl \triangleq ( \underbrace{\avg+\varepsilon}_{\mathrm{indices\ }1 \mathrm{to\ } \ell},\mu_{\ell+1},\dots
  ,\underbrace{\avg}_{\mathrm{index\ }k},\dots,\mu_K )\transpose.
\]
Observe that since $\ell$ depends on $t^*$, we can't directly compute $t^*$ from the expression above. Instead, we use the fact that $\ell$ is unique by definition. Therefore, to determine $t^*$ one can compute $\avg$ for all values of $\ell \in [|1,k-1|]$ and search for the index $\ell$ satisfying $\mu_{\ell+1} -\varepsilon <\avg \leq \mu_\ell - \varepsilon$ and with minimum value in the objective (\ref{eq:CT_inf}).



As a summary, we have reduced the minimization problem over the infinite set $\Alt\bmu$ to a combinatorial search over a finite number of alternative bandit instances whose analytical expression is given in the next definition.
\begin{definition}
  Let $\blambdakl$ be a vector created form $\bmu$ by replacing elements on positions $k$ and $\ell$ (resp. 1 to $\ell$), defined as:
  \[
    \blambdakl \triangleq (\mu_1,\dots,\underbrace{\avg}_{\mathrm{index\ }k},\dots,\underbrace{\avg+\varepsilon}_{\mathrm{index\ }\ell},\dots,\mu_K )\transpose
  \]
  for $k\in\G{\bmu}$ and
  \[
    \blambdakl \triangleq ( \underbrace{\avg+\varepsilon}_{\mathrm{indices\ }1 \mathrm{to\ } \ell},\mu_{\ell+1},\dots
    ,\underbrace{\avg}_{\mathrm{index\ }k},\dots,\mu_K )\transpose
  \]
  for $k\in\B{\bmu}$ where $\avg$ is a weighted average of elements on positions $k$ and $\ell$ (resp. 1 to $\ell$) defined as:
  \[
    \avg \triangleq \frac{\omega_k\mu_k + \omega_\ell(\mu_\ell-\varepsilon)}{\omega_k + \omega_\ell}
  \]
  for $k\in\G{\bmu}$ and
  \[
    \avg \triangleq  \frac{\omega_k\mu_k + \sum_{a = 1}^\ell \omega_a(\mu_a-\varepsilon)}{\omega_k + \sum_{a = 1}^\ell \omega_a}
  \]
  for $k\in\B{\bmu}$.
 \label{def:1}
\end{definition}
The next lemma then states that the best response oracle belongs to the finite set of $(\blambdakl)_{k,\ell}$.
\begin{lemma}\label{lem:bestresponse}
  Using the previous definition, $\blambdastar$ can be computed as
  \[
    \blambdastar = \argmin_{\blambda\in \Lambda_G\cup\Lambda_B} \suma \omega_a\frac{(\mu_a - \lambda_a)^2}{2}
  \]
  where
  \[
    \Lambda_G = \{\blambdakl : k\in\G{\bmu}, \ell\in \G{\bmu}/\{k\}\}
  \]
  and
  \begin{align*}
    \Lambda_B = \{&\blambdakl : k  \in\B{\bmu}, \ell \in [|1,k-1|]\\         &\mathrm{s.t.}\ \mu_\ell \ge \avg + \varepsilon > \mu_{\ell+1}\}.
  \end{align*}
\end{lemma}


\section{Solving the max-min problem: Optimal weights}\label{sec:max_min}

First observe that we can rewrite $T_{\varepsilon}(\bmu,.)^{-1}$ as a minimum of linear functions:

\begin{equation}
  T_{\varepsilon}(\bmu,\bomega)^{-1} = \inf_{\bd \in \Dset} \bomega\transpose\bd
  \label{eq:def_infimum}
\end{equation}
where
\[
  \Dset \triangleq \left\{\left(\frac{(\lambda_a - \mu_a)^2}{2}\right)_{a\in[K]}^{\transpose}\ \big|\ \blambda \in \Alt\bmu \right\}.
\]
Note that by using $\Dset$ instead of $\Alt{\bmu}$, the optimization function becomes simpler for the price of more complex domain (see Figure \ref{fig:sets} for an example). As a result, $T_{\varepsilon}(\bmu,.)^{-1}$ is concave and we can compute its supergradients thanks to Danskin's Theorem \cite{Danskin} which we recall in the lemma below.


\begin{figure}[h]
  \centering
  \includegraphics[width=.49\textwidth]{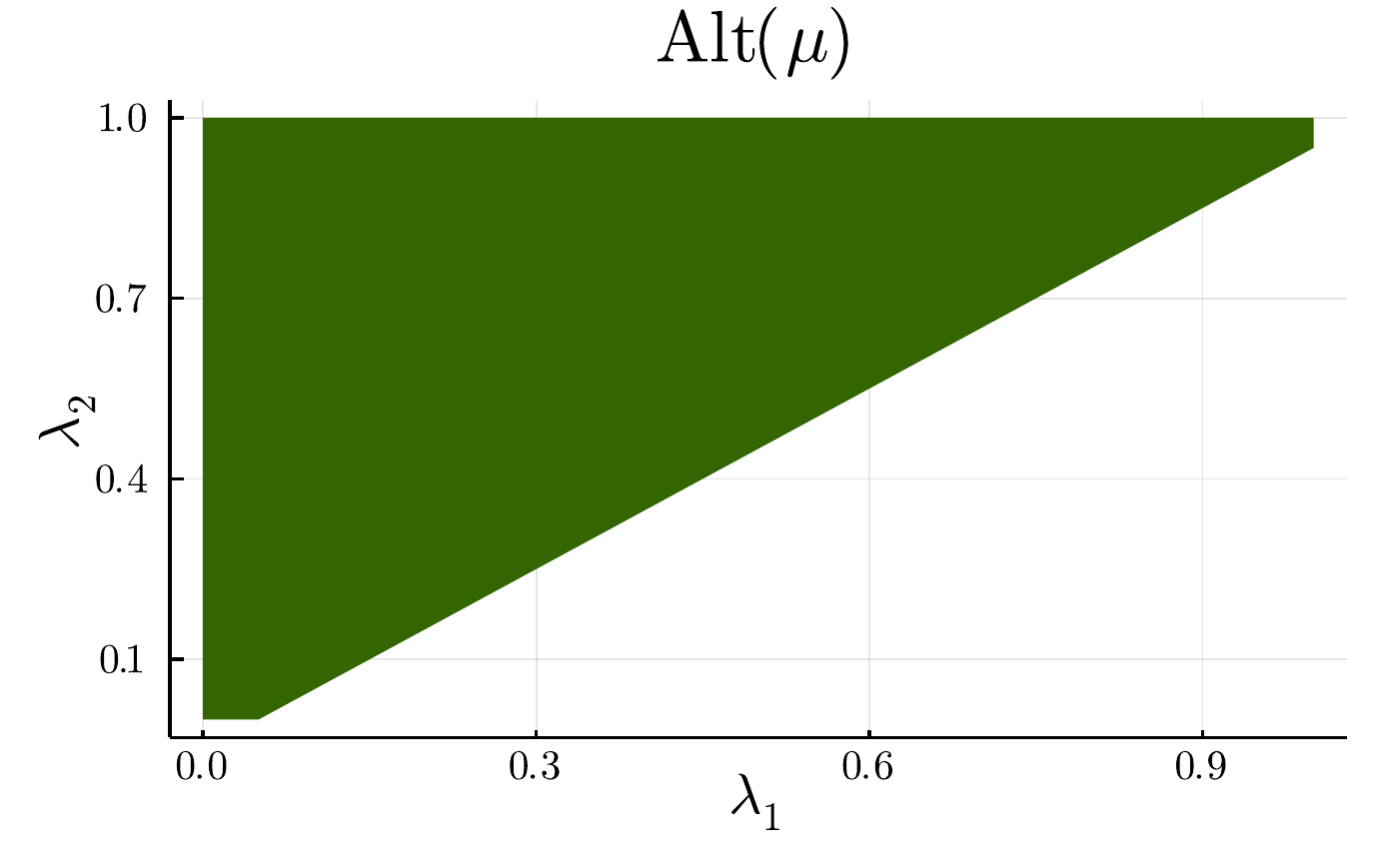}
  \includegraphics[width=.49\textwidth]{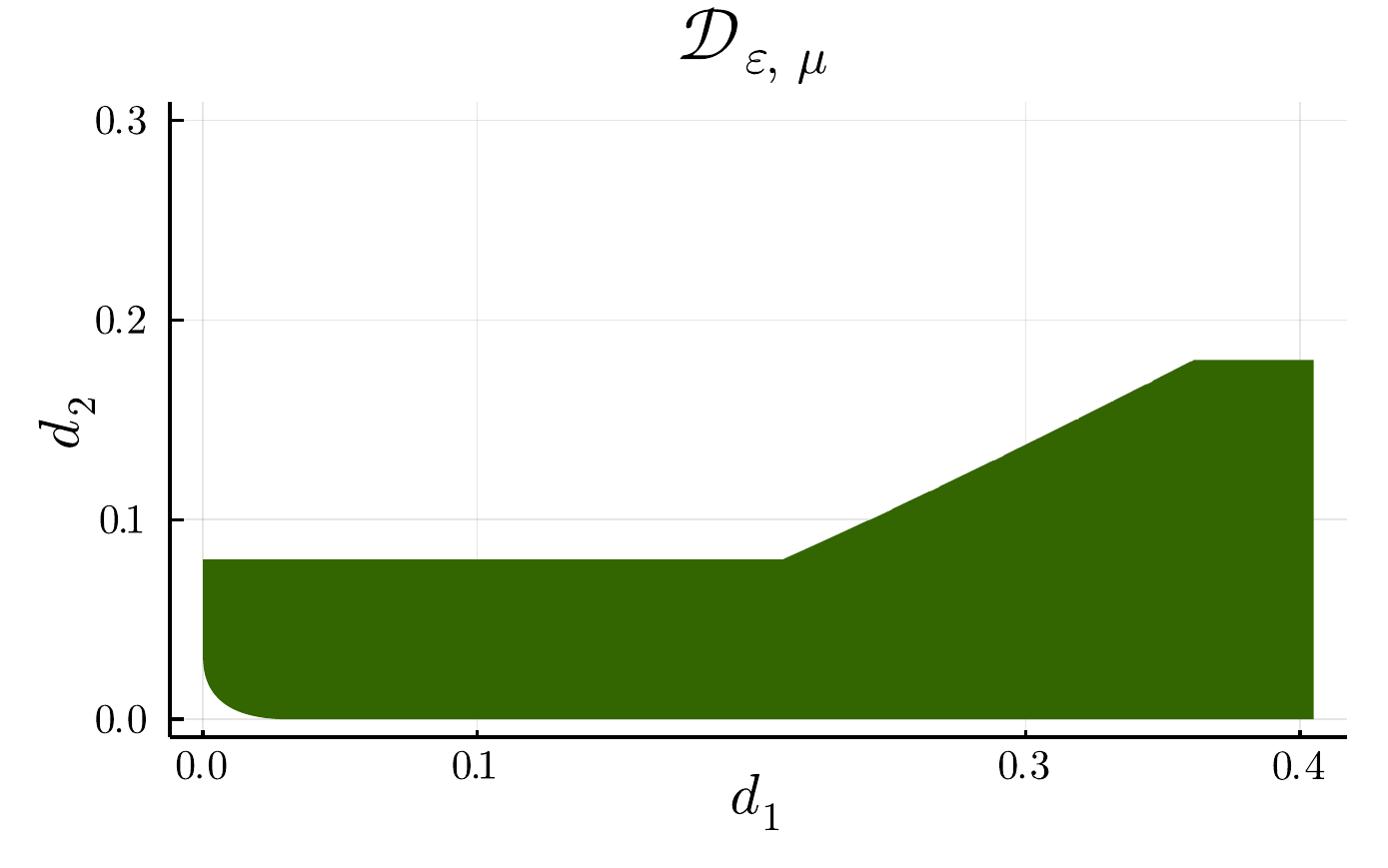}
  \caption{Comparison of $\Alt{\bmu}$ with Simple Linear Boundaries (First Figure) and $\Dset$ with Non-Linear Boundaries (Second Figure) for $\bmu = [0.9,\,0.6]$ and $\epsilon=0.05$. }
  \label{fig:sets}
\end{figure}

\begin{lemma}(Danskin's Theorem)
  Let $\blambda^*(\bomega)$ be a best response to $\bomega$ and define $\bd^*(\bomega) \triangleq \big(\frac{(\blambda^*(\bomega)_a - \mu_a)^2}{2}\big)_{a\in[K]}^{\transpose}$. Then $\bd^*(\bomega)$ is a supergradient of $T_{\varepsilon}(\bmu,.)^{-1}$ at $\bomega$.
\end{lemma}



Next we prove that $T_{\varepsilon}(\bmu,.)^{-1}$ is Liptschiz.
\begin{lemma}\label{lem:lipsch}
  The function $\bomega \mapsto T_{\varepsilon}(\bmu,\bomega)^{-1}$ is $L$-Lipschitz with respect to $\norm{\,\cdot\,}_1$ for any
  \[
    L\ge\max_{a,b\in[K]}\frac{(\mu_a-\mu_b + \varepsilon)^2}{2}\;.
  \]
\end{lemma}
\begin{proof}
  As we showed in Lemma \ref{lem:bestresponse}, the best response $\blambdastar$ to $\bomega$ is created from $\bmu$ by replacing some of the elements by $\avg$ or $\avg + \varepsilon$. We also know that $\avg$ is a weighted average of an element of $\bmu$ with one or more elements of $\bmu$ decreased by $\varepsilon$. This means that:
  \[
    \max_{a\in[K]}\mu_a \ge \avg \ge \min_{a\in[K]}\mu_a - \varepsilon
  \]
  and, as a consequence, we have:
  \[
    |\mu_i - \blambdastar_i| \le \max_{a,b\in[K]}(\mu_a-\mu_b + \varepsilon)
  \]
  for any $i\in[K]$. Let $f(\bomega) \triangleq T_{\varepsilon}(\bmu,\bomega)^{-1}$. Using the last inequality and the definition of $\bd^*(\bomega)$, we can obtain:
  \begin{align*}
    f(\bomega)-f(\bomega') & \leq (\bomega - \bomega')\transpose \bd^*(\bomega')                                \\
                           & \leq \|\bomega-\bomega'\|_1 \|\bd^*(\bomega')\|_\infty                             \\
                           & \leq \|\bomega-\bomega'\|_1\max_{a,b\in[K]}\frac{(\mu_a-\mu_b + \varepsilon)^2}{2}
  \end{align*}

  for any $\bomega,\,\bomega'\in\simplex{K}$.
\end{proof}

As a summary $T_{\varepsilon}(\bmu,.)^{-1}$ is concave, Lipschitz and we have a simple expression to compute its supergradients through the best response oracle. Therefore we have all the necessary ingredients to apply a gradient-based algorithm in order to find the optimal weights and therefore, the value of $\CT$. The algorithm of our choice is the mirror ascent algorithm which provides the following guarantees:

\begin{proposition}\cite{bubeck2015}
  Let $\bomega_1 = (\frac{1}{K},\dots,\frac{1}{K})\transpose$ and learning rate $\alpha_n = \frac{1}{L}\sqrt{\frac{2\log K}{n}}$. Then using mirror ascent algorithm to maximize a $L$-Lipschitz function $f$, with respect to $\norm{\,\cdot\,}_1$, defined on $\Delta_K$ with generalized negative entropy $\Phi(\bomega) = \suma \omega_a \log(\omega_a)$ as the mirror map enjoys the following guarantees:
  \[
    f(\bomega^*) - f\left(\frac{1}{N}\sum_{n=1}^N\bomega_n\right) \le L\sqrt{\frac{2\log K}{N}}\;.
  \]
\end{proposition}

\subsubsection{Computational complexity of our algorithm.} To simplify the presentation and analysis, we chose to focus on the vanilla version of Track and Stop. However, in practice this requires solving the optimization program that appears in the lower bound at every time step, which can result in large run times. Nonetheless, we note that there are many possible adaptations of Track and Stop that reduce the computational complexity, while retaining the guarantees of asymptotic optimality in terms of the sample complexity (and with a demonstrated small performance loss experimentally). A first solution is to use Franke-Wolfe style algorithms \cite{Menard2019Gradient,Wang2021FW}, which only perform a gradient step of the optimization program at every step. Another adaptation is the Lazy Track-and-Stop \cite{Jedra2020Linear}, which updates the weights that are tracked by the algorithms every once in a while. We chose the latter solution in our implementation, where we updated the weights every $100K$ steps.

\section{Experiments}\label{sec:experiments}
We conducted three experiments to compare Track-and-Stop with state-of-the-art algorithms, mainly $(\mathrm{ST})^2$ and \textrm{FAREAST} from \cite{Mason2020}. In the first experiment, we simulate a multi-armed bandit with Gaussian rewards of means $\bmu = [1, 1, 1, 1, 0.05]$, variance one and a parameter $\varepsilon = 0.9$.
We chose this particular instance $\bmu$ because its difficulty is two-fold: First, the last arm $\mu_5$ is very close to the threshold $\max_a \mu_a - \varepsilon$. Second, the argmax is realized by more than one arm, which implies that any algorithm must estimate all the means to high precision to produce a confident guess of $\G\bmu$. Indeed, a small underestimation error of $\max_a \mu_a$ would mean wrongly classifying $\mu_5$ as a good arm. We run the three algorithms for several values of $\delta$ ranging from $\delta = 0.1$ to $\delta = 10^{-10}$, with $N=100$ Monte-Carlo simulations for each risk level. Figure \ref{fig:simulation} shows the expected stopping time along with the $10\%$ and $90\%$ quantiles (shaded area) for each algorithm. Track-and-Stop consistently outperforms
$(\mathrm{ST})^2$ and \textrm{FAREAST}, even for moderate values of $\delta$. Also note that, as we pointed out in Remark 1, the sample complexity of Track-and-Stop is within some multiplicative constant of $(\mathrm{ST})^2$.


\begin{figure}[ht!]
    \centering
    \includegraphics[width=0.5\textwidth]{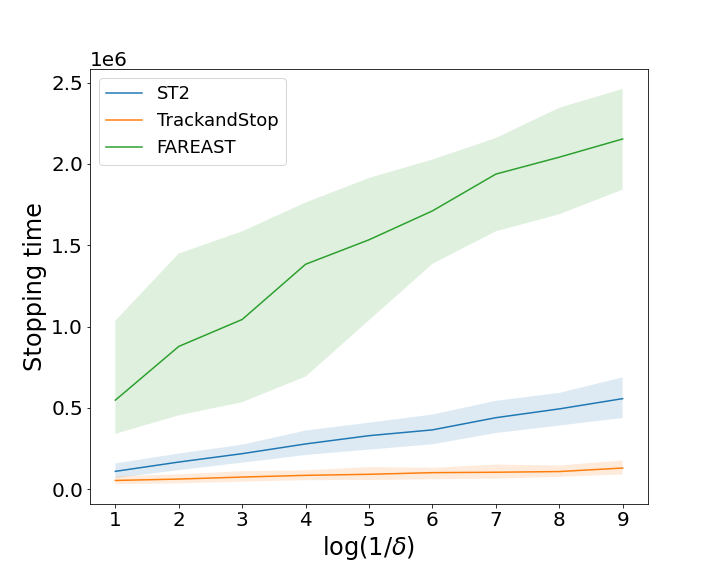}
    \hfill
    \includegraphics[width=0.49\textwidth]{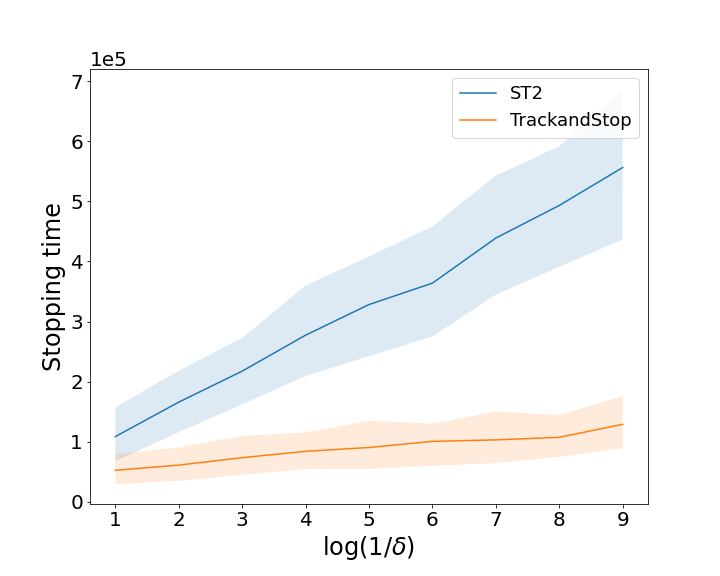}
    \caption{Expected Stopping Time on $\bmu = [1, 1, 1, 1, 0.05]$. Left: All three Algorithms. Right: Track-and-Stop vs FAREAST.}
    \label{fig:simulation}
\end{figure}
Next, we examine the performance of the algorithms w.r.t the number of arms. For any given $K$, we consider a bandit problem $\mu$ similar to the previous instance: $\forall a \in [|1, K-1|],\ \mu_a = 1$ and $\mu_K = 0.05$. We fix $\varepsilon = 0.9$ and $\delta = 0.1$ and run $N = 30$ Monte-Carlo simulations for each $K$. Figure \ref{fig:Narms} shows, in log-scale, the ratio of the sample complexities of $(\mathrm{ST})^2$ and FAREAST w.r.t to the sample complexity of Track-and-Stop. We see that Track-and-Stop performs better than $(\mathrm{ST})^2$ (resp. FAREAST) for small values of $K$. However when the number of arms grows larger than $K = 40$ (resp. $K = 60$), $(\mathrm{ST})^2$ (resp. FAREAST) have a smaller sample complexity.
\begin{figure}[ht!]
    \centering
    \includegraphics[width=0.49\textwidth]{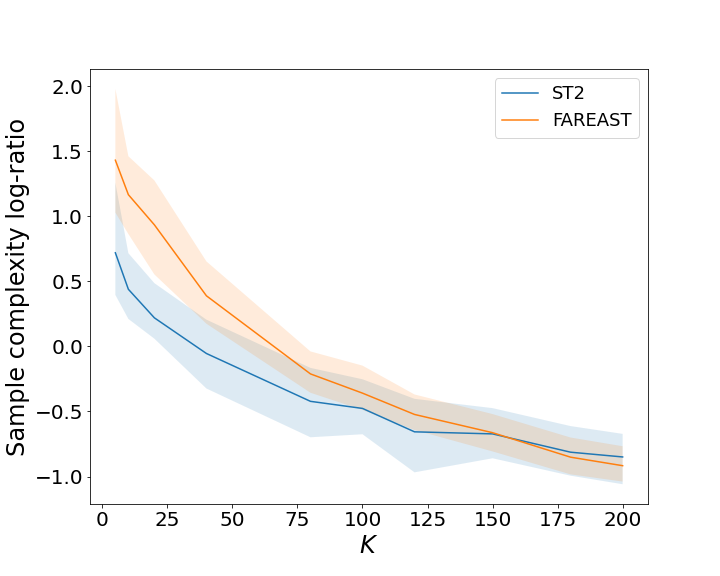}
    \includegraphics[width =.49\textwidth]{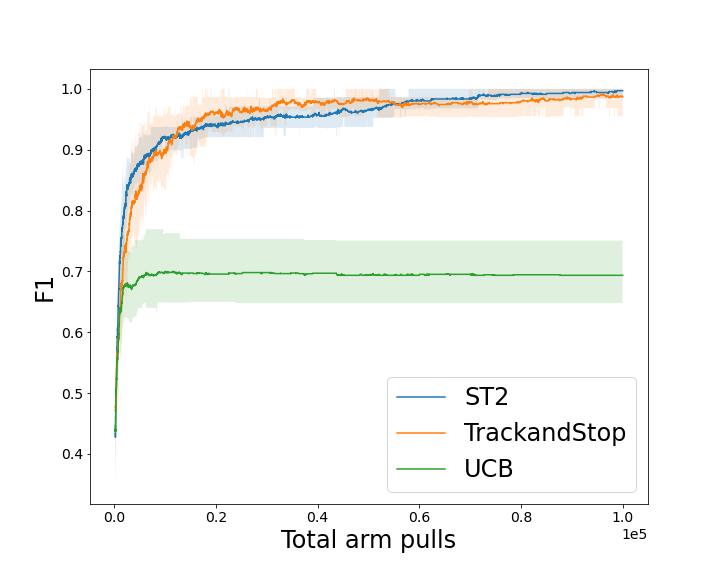}
    \caption{Left: $\displaystyle{\log_{10}\big(\E_{\mathrm{Alg}}[\tau_\delta]/\E_{\mathrm{TaS}}[\tau_\delta]\big)}$ for $\mathrm{Alg} \in \{ (\mathrm{ST})^2, \mathrm{FAREAST}\}$ and $\mathrm{TaS} =$ Track-and-Stop, $K_{\min} = 5$ arms. Right: F1 scores for Cancer Drug discovery.}
    \label{fig:Narms}
\end{figure}
\paragraph{} Finally, we rerun the Cancer Drug Discovery experiment from \cite{Mason2020}. Note that this experiment is more adapted to a \textit{fixed budget setting} where we fix a sampling budget and the algorithm stops once it has reached this limit, which is different from the \textit{fixed confidence} setting that our algorithm was designed for. The goal is to find, among a list of $189$ chemical compounds, potential inhibitors to \textbf{ACRVL1}, a Kinaze that researchers \cite{cancer19} have linked to several forms of cancer. We use the same dataset as \cite{Mason2020}, where for each compound a percent control\footnote{percent control is a metric expressing the efficiency of the compound as an inhibitor against the target Kinaze.} is reported. We fix a budget of samples $N = 10^5$ and try to find all the $\varepsilon$-good compounds in the multiplicative case with $\varepsilon = 0.8$. For each algorithm, we compute the F1-score\footnote{F1 score is the harmonic mean of precision (the proportion of arms in $\widehat{G}$ that are actually good) and recall (the proportion of arms in $\G\bmu$ that were correctly returned in $\widehat{G}$).} of its current estimate $\widehat{G}_\varepsilon = \{i\ :\ \widehat{\mu_i} \geq (1-\varepsilon)\max_a \widehat{\mu}_a \}$ after every iteration. The F1-score in this fixed-budget setting reflects how good is the sampling scheme of an algorithm, independently of its stopping condition. In Figure \ref{fig:Narms} we plot the average F1-score along with the $10\%$ and $90\%$ quantiles (shaded area). We see that $(\mathrm{ST})^2$ and Track-and-Stop have comparable performance and that both outperform UCB's sampling scheme.



\section{Conclusion}
We shed a new light on the sample complexity of finding all the $\varepsilon$-good arms in a multi-armed bandit with Gaussian rewards. We derived two lower bounds, identifying the characteristic time that reflects the true hardness of the problem in the asymptotic regime. Moreover, we proved a second bound highlighting an additional cost that is linear in the number of arms and can be arbitrarily larger than the first bound for moderate values of the risk. Then, capitalizing on an algorithm solving the optimization program that defines the characteristic time, we proposed an efficient Track-and-Stop strategy whose sample complexity matches the lower bound for small values of the risk level. Finally, we showed through numerical simulations that our algorithm outperforms state-of-the-art methods for bandits with small to moderate number of arms. Several directions are worth investigating in the future. Notably, we observe that while Track-and-Stop performs better in the fixed-$K$-small-$\delta$ regime, the elimination based algorithms (ST)$^2$ and FAREAST become more efficient in the large-$K$-fixed-$\delta$ regime. It would be interesting to understand the underlying tradeoff between the number of arms and confidence parameter. This will help design pure exploration strategies having best of both worlds guarantees.

\bibliographystyle{alpha}
\bibliography{library}

\newpage
\appendix

\appendix
\onecolumn
\setcounter{theorem}{6}

\section{Multiplicative Case}\label{sec:multiplicative}
Most of the  previous analysis of the additive case can be directly applied to the multiplicative case where the set of the best arms is defined as $\G{\bmu} \triangleq \{ a\in[K]:\mu_a \geq (1 - \varepsilon) \max_{i} \mu_i \}$. The only missing pieces are the best response oracle and an upper bound on the Lipschitz constant for a mirror ascent algorithm. Solving these two problems gives us all the necessary tools to produce the multiplicative version of the algorithm presented in \ref{main_algo}.
\subsection{Best Response Oracle - Multiplicative Case}
Similarly, as in the additive case, split the reasoning into two parts: making a good arm bad or making a bad arm good. This provides the following variation of Lemma 3
\begin{lemma}\label{lem:mul_response}
  Define $\blambdakl$ as in Definition \ref{def:1} for
  \[
    \avg \triangleq \frac{(1-\epsilon)^2\omega_k\mu_k + (1-\epsilon)\omega_\ell\mu_\ell}{(1-\epsilon)^2\omega_k + \omega_\ell}
  \]

  in the case of $k\in\G{\bmu}$ and

  \[
    \avg \triangleq \frac{(1-\epsilon)^2\omega_k\mu_k + \suml(1-\epsilon)\omega_a\mu_a}{(1-\epsilon)^2\omega_k + \suml\omega_a}
  \]

  in the case of $k\in\B{\bmu}$. Then the best response $\blambdastar$ to a given $\bomega$ is defined as:
  \[
    \blambdastar = \argmin_{\blambda\in \Lambda_G\cup\Lambda_B} \suma \omega_a\frac{(\mu_a - \lambda_a)^2}{2}.
  \]
  where
  \[
    \Lambda_G = \{\blambdakl : k\in\G{\bmu}, \ell\in \G{\bmu}/\{k\}\}
  \]
  and
  \begin{align*}
    \Lambda_B = \{&\blambdakl : k  \in\B{\bmu} , \ell \in [|1,k-1|]\\            &\mathrm{s.t.}\ \mu_\ell \ge \avg/(1-\epsilon) > \mu_{\ell+1}\}.
  \end{align*}
\end{lemma}

\begin{proof}\phantom{x}\\
  \textbf{Case 1: Making one of the $\epsilon$-optimal arms bad.}
  Selecting a good arm $k\in\G{\bmu}$ and any other arm $\ell\in\G{\bmu}$ we can set the expectation of arm $k$ to $t$ and the expectation of arm $\ell$ to $t/(1-\epsilon)$. Using calculus, we can find the optimal value
  \[
    t = \frac{(1-\epsilon)^2\omega_k\mu_k + (1-\epsilon)\omega_\ell\mu_\ell}{(1-\epsilon)^2\omega_k + \omega_\ell}.
  \]

  \textbf{Case 2: Making one of the sub-optimal arms good.}
  Selecting a bad arm $k\in\B{\bmu}$ and the first $\ell$ arms we can set the expectation of arm $k$ to $t$ and the expectation of the first $\ell$ arms to $t/(1-\epsilon)$. Note that $\ell$ is a unique integer such that $\mu_\ell\ge t/(1-\epsilon) > \mu_{\ell+1}$ thanks to the same argument as in the additive case. Using calculus, we can find the optimal value
  \[
    t = \frac{(1-\epsilon)^2\omega_k\mu_k + \suml(1-\epsilon)\omega_a\mu_a}{(1-\epsilon)^2\omega_k + \suml\omega_a}.
  \]

\end{proof}

\subsection{Upper bound on the Lipschitz constant}
The final ingredient is a constant $L$ for a mirror ascent algorithm, characterized in the following lemma.
\begin{lemma}
  The function $\bomega \mapsto T_{\varepsilon}(\bmu,\bomega)^{-1}$ is $L$-Lipschitz with respect to $\norm{\,\cdot\,}_1$ for any
  \[
    L\ge\max_{a,b\in[K]}\frac{(\mu_a-\mu_b(1-\epsilon))^2}{2(1-\epsilon)^2}\;.
  \]
\end{lemma}

\begin{proof}
  Lemma \ref{lem:mul_response} shows that the expectation of the selected arm $k$ should be in equal to
  \[
    t = \frac{(1-\epsilon)^2\omega_k\mu_k + \sum(1-\epsilon)\omega_a\mu_a}{(1-\epsilon)^2\omega_k + \sum\omega_a}
  \]
  where the sum is taken either over a single arm $\ell$ or over the first $\ell$ arms. In both cases we can show
  \[
    (1-\epsilon)\mu_\mathrm{min} \le \frac{(1-\epsilon)^2\omega_k\mu_k + \sum(1-\epsilon)\omega_a\mu_a}{(1-\epsilon)^2\omega_k + \sum\omega_a}
  \]
  using $(1-\epsilon) < 1$, $\mu_\mathrm{min} \le \mu_a$ for any $a\in[K]$ and
  \[
    \frac{(1-\epsilon)^2\omega_k\mu_k + \sum(1-\epsilon)\omega_a\mu_a}{(1-\epsilon)^2\omega_k + \sum\omega_a}\le \mu_\mathrm{max}
  \]
  using $(1-\epsilon) < 1$ and $\mu_\mathrm{max} \ge \mu_a$ for any $a\in[K]$. Using this inequality we can show that $(\mu_a - \lambda_a)^2$ can be upper-bounded by
  \[
    \left(\frac{\mu_\mathrm{max}}{1-\epsilon} - \mu_\mathrm{min}\right)^2
  \]
  for any best response $\blambda$. Applying this bound to the optimization problem concludes the proof of the lemma.
\end{proof}

\section{Proof of Proposition 1}
Let $\delta\in(0,1)$ be a confidence parameter and $\bmu$ be a bandit problem. Considering a $\delta$-PAC strategy, we can apply transportation Lemma 1 of \cite{kaufmann2015} to obtain
\[
  \suma \frac{(\mu_a-\lambda_a)^2}{2}\E[N_a] \ge \kl(\delta,\,1-\delta)
\]
for $N_a \triangleq N_a(\tau_\delta)$. This holds for any alternative problem $\blambda\in\Alt{\bmu}$. Therefore, we have
\begin{align*}
  \kl(\delta,\,1-\delta) & \le \inf_{\blambda\in\Alt{\bmu}} \suma \E[\tau_\delta]\frac{(\mu_a-\lambda_a)^2}{2}\frac{\E[N_a]}{\E[\tau_\delta]}       \\
                         & \le \E[\tau_\delta]\sup_{\bomega\in\simplex{K}}\inf_{\blambda\in\Alt{\bmu}} \suma\omega_a \frac{(\mu_a-\lambda_a)^2}{2}.
\end{align*}
Fraction $\E[N_a] / \E[\tau_\delta]$ represents empirical proportion of arm draws which enables us to upper-bound it by a supremum over all possible arm allocations. We conclude by noting that $ \kl(\delta,\,1-\delta) \geq \log(1/2.4\delta)$.

\section{Proof of Theorem \ref{thm:simulator_LB}}\label{sec:proof_simulator}
We restrict our attention to symmetric algorithms. An algorithm $\cA$ is said to be symmetric if it satisfies for any permutation $\pi$, any integer $n\geq 1$ and any sequence of actions $A_1,\ldots, A_n$,
\begin{align*}
    \P_{\cA, \nu}\big((a_1,\ldots, a_n) = (A_1, \dots, A_n) \big) = \P_{\cA, \pi(\nu)}\big((a_1,\ldots, a_n) = (\pi(A_1), \dots, \pi(A_n)) \big).
\end{align*}
In words, $\cA$ is symmetric if it is indifferent to the order of the arms and acts only based on the underlying distributions. \cite{simchowitz17a} showed that for any algorithm $\cA$, one can easily build a symmetrized version $\cA_{sym}$ such that for any bandit instance $\nu$, $\E_{\pi \sim \mathbf{S}_K}\E_{\cA, \pi(\nu)}[\tau_\delta] = \E_{\cA_{sym}, \nu}[\tau_\delta]$. Therefore, we only need to prove the lower bound for symmetric algorithms.\\
Throughout the proof it will be useful to represent bandit instances using the {\it random table model} \cite{LS19}: $\nu$ can be defined as a collection of random variables $(X_{a,t})_{a\in [K], t\geq 1}$ where $X_{a,t}$ represents the reward received when playing arm $a$ for the $n$-th time. Therefore it is enough to specify the law of each $X_{a,t}$ to define $\nu$.\\
The lemma below shows that no arm can be played much less than the arms in $G_{\beta_\epsilon}(\bmu)$ and enables us to avoid performing algorithmic reductions of the All-$\epsilon$ problem to Best Arm Identification or $\beta$-isolated tests as was done in \cite{Mason2020}.
\begin{lemma}
For all arms $b \in [K]$ and all integers $n\geq 1$,
\begin{align*}
    \frac{1}{|G_{\beta_\epsilon}(\bmu)|}\sum_{a\in G_{\beta_\epsilon}(\bmu)} \P_{\nu}\left( N_a(\tau) > n \right) -  (\mu_1 - \mu_b)\sqrt{n/2} \leq 3 \P_\nu(N_b(\tau) > n).
\end{align*}
\label{lemma:simulator_1}
\end{lemma}
\begin{proof}
Fix $a\in G_{\beta_\epsilon}(\bmu)\setminus\{b\}$ and $n\geq 1$. Let $\pi$ be the permutation that swaps arms $a$ and $b$, i.e. $\pi(a) = b, \pi(b) = a$ and $\pi(k) = k$ for $k \in [K]\setminus\{a,b\}$. We define the non-stationary bandit instances $\widetilde{\nu}$ and $\widetilde{\pi}$ such that
\begin{align*}
    \widetilde{\nu} : 
\begin{tabular}{c||c|c|}
\text{Arm } & First $n$ rewards & Next rewards \\ 
\hline \hline
a  & $\sim\cN(\mu_a,1)$ & $\sim\cN(\mu_a,1)$  \\
\hline
b & $\sim\cN(\mu_b,1)$ & $\sim\cN(\mu_a,1)$ \\
\hline
$k\in [K]\setminus\{a,b\}$ & $\sim\cN(\mu_k,1)$   & $\sim\cN(\mu_k,1)$\\
\hline
\end{tabular} 
\quad \text{and} \quad
\end{align*}
\begin{align*}
\widetilde{\pi} : 
\begin{tabular}{c||c|c|}
\text{Arm } & First $n$ rewards & Next rewards \\ 
\hline \hline
a & $\sim\cN(\mu_b,1)$ & $\sim\cN(\mu_a,1)$  \\
\hline
b & $\sim\cN(\mu_a,1)$ & $\sim\cN(\mu_a,1)$ \\
\hline
$k\in [K]\setminus\{a,b\}$ & $\sim\cN(\mu_k,1)$   & $\sim\cN(\mu_k,1)$\\
\hline
\end{tabular}
\end{align*}

$\widetilde{\nu}$ and $\widetilde{\pi}$ will only serve as "bridges" in our change-of-measure argument. In particular, we do not require that the algorithm returns a good answer on any of them. Let $\P_\lambda$ denote the law of all relevant random variables (rewards, actions played, stopping times..) when running algorithm $\cA$ on instance $\lambda$ and define the event $E = (N_b(\tau) \leq n)$. Observe that $\P_{\nu}\left( E \cap . \right) = \P_{\widetilde{\nu}}\left( E \cap . \right)$, since under $E$ algorithm $\cA$ observes the same distribution of rewards. Thus using Bayes' Theorem one can write
\begin{align}
 \TV{\P_{\widetilde{\nu}}, \P_\nu} &= \mathrm{TV}\big(\P_{\widetilde{\nu}}(E)\times\P_{\widetilde{\nu}}(.|E) + \P_{\widetilde{\nu}}(E^c)\times \P_{\widetilde{\nu}}(.|E^c),\ \P_\nu(E)\times\P_\nu(.|E)+\P_\nu(E^c)\times\P_\nu(.|E^c)\big)\nonumber\\
 &\overset{(a)}{=}  \mathrm{TV}\big(\P_\nu(E)\times\P_{\widetilde{\nu}}(.|E) + \P_\nu(E^c) \times\P_{\widetilde{\nu}}(.|E^c),\ \P_\nu(E)\times\P_\nu(.|E)+\P_\nu(E^c)\times\P_\nu(.|E^c)\big)\nonumber\\
 &\overset{(b)}{\leq}  \P_\nu(E)\mathrm{TV}\big(\P_{\widetilde{\nu}}(.|E),\ \P_\nu(.|E)\big) + \P_\nu(E^c)\mathrm{TV}\big(\P_{\widetilde{\nu}}(.|E^c),\ \P_\nu(.|E^c)\big)\nonumber\\
 &\overset{(c)}{\leq} \P_\nu(E^c) = \P_{\nu}\left( N_b(\tau) > n \right),
\label{eq:TV_decomposition_1}
\end{align}
where (a) is because $\P_\nu(E) = \P_{\widetilde{\nu}}(E)$ hence $\P_\nu(E^c) = \P_{\widetilde{\nu}}(E^c)$ also, (b) is by the joint convexity of the TV distance and (c) is because $\P_{\nu}\left( E \cap . \right) = \P_{\widetilde{\nu}}\left( E \cap . \right)$ implies that $\P_{\nu}\left(.| E \right) = \P_{\widetilde{\nu}}\left(.| E \right)$. Similarly, by considering event $E' = (N_b(\tau)>n)$ we have
\begin{align}
    \TV{\P_{\pi(\nu)}, \P_{\widetilde{\pi}}} \leq \P_{\pi(\nu)}\left( N_a(\tau) > n \right).
\label{eq:TV_decomposition_2}
\end{align}
Using the above, one can write 
\begin{align*}
    \P_{\nu}\left( N_a(\tau) > n \right) - \P_{\pi(\nu)}\left( N_a(\tau) > n \right)  &\leq \TV{\P_\nu, \P_{\pi(\nu)}}\\
    &\leq \TV{\P_{\pi(\nu)}, \P_{\widetilde{\pi}}} + \TV{\P_{\widetilde{\pi}}, \P_{\widetilde{\nu}}} + \TV{\P_{\widetilde{\nu}}, \P_\nu}\\
    &\overset{(a)}{\leq} \P_{\pi(\nu)}(N_a(\tau) > n) + \P_\nu(N_b(\tau) > n) + \sqrt{\frac{\KL(\P_{\widetilde{\pi}}, \P_{\widetilde{\nu}})}{2}}\\
    &\overset{(b)}{\leq} 2\P_\nu(N_b(\tau) > n) + \sqrt{\frac{\KL(\P_{\widetilde{\pi}}, \P_{\widetilde{\nu}})}{2}},\\
\end{align*}
where (a) comes from combining (\ref{eq:TV_decomposition_1}) and (\ref{eq:TV_decomposition_2}) and using Pinsker's inequality and (b) is because the symmetry of the algorithm implies that $\P_{\pi(\nu)}(N_a(\tau) > n) = \P_\nu(N_b(\tau) > n)$. Now observe that $\KL(\P_{\widetilde{\pi}}, \P_{\widetilde{\nu}}) = n \big(\KL(\cN(\mu_a,1), \cN(\mu_b,1))+ \KL(\cN(\mu_b,1), \cN(\mu_a,1))\big) = n (\mu_a - \mu_b)^2 \leq n (\mu_1 - \mu_b)^2$. So that the inequality above simplifies to
\begin{align*}
   \P_{\nu}\left( N_a(\tau) > n \right) - (\mu_1 - \mu_b)\sqrt{n/2} \leq 3\P_\nu(N_b(\tau) > n).
\end{align*}
Note that the inequality above holds trivially when $a =b$. Now for a fixed $b$ by summing the inequality over all arms $a\in G_{\beta_\epsilon}(\bmu)$ we get
\begin{align*}
    \sum_{a\in G_{\beta_\epsilon}(\bmu)} \P_{\nu}\left( N_a(\tau) > n \right) - |G_{\beta_\epsilon}(\bmu)| (\mu_1 - \mu_b)\sqrt{n/2} \leq 3|G_{\beta_\epsilon}(\bmu)| \P_\nu(N_b(\tau) > n).
\end{align*}
Hence the statement of the lemma.
\end{proof}
Now the next lemma will show that arms in $G_{\beta_\epsilon}(\bmu)$ must be played $\Omega(1/\beta_\epsilon^2)$ times because underestimating their means by $\beta_\epsilon$ may cause to declare a bad arm as $\epsilon$-optimal. 
\begin{lemma}
For all integers $n\geq1$,
\begin{align*}
    1-2\delta - |G_{\beta_\epsilon}(\bmu)| \beta_\epsilon\sqrt{n}/2 \leq \sum_{a\in G_{\beta_\epsilon}(\bmu)} \P_{\nu}\left( N_a(\tau) > n \right)
\end{align*}
\label{lemma:simulator_2}
\end{lemma}
\begin{proof}
Let $\eta>0$. We define the instances $\lambda$ and $\widetilde{\nu}$ such that 
\begin{align*}
    \lambda : 
\begin{tabular}{c||c|c|}
\text{Arm } & All rewards \\ 
\hline \hline
For $a\in G_{\beta_\epsilon}(\bmu)$  & $\sim\cN(\mu_1-\beta_\epsilon - \eta,1)$  \\
\hline
For $k\in [K]\setminus G_{\beta_\epsilon}(\bmu)$  & $\sim\cN(\mu_k,1)$\\
\hline
\end{tabular}  
\quad \text{and} \quad
\end{align*}
\begin{align*}
    \widetilde{\nu} : 
\begin{tabular}{c||c|c|}
\text{Arm } & First $n$ rewards & Next rewards \\ 
\hline \hline
For $a\in G_{\beta_\epsilon}(\bmu)$  & $\sim\cN(\mu_a,1)$ & $\sim\cN(\mu_1-\beta_\epsilon - \eta,1)$  \\
\hline
For $k\in [K]\setminus G_{\beta_\epsilon}(\bmu)$ & $\sim\cN(\mu_k,1)$   & $\sim\cN(\mu_k,1)$\\
\hline
\end{tabular} 
\end{align*}
By considering the event $E = (\forall a\in G_{\beta_\epsilon}(\bmu),\ N_a(\tau) \leq n)$, one can show in a similar fashion to the proof of Lemma \ref{lemma:simulator_1} that 
\begin{align}
    \TV{\P_{\widetilde{\nu}}, \P_\nu} \leq \P_\nu(\exists a\in G_{\beta_\epsilon}(\bmu),\ N_a(\tau) > n)  \leq \sum_{a\in G_{\beta_\epsilon}(\bmu)} \P_{\nu}\left( N_a(\tau) > n \right).
\label{eq:TV_decomposition_3}
\end{align}
Let $B \in \argmin\limits_{k\notin\G\bmu} \mu_1 -\varepsilon - \mu_k$ be a bad arm with the highest mean. Note that $B$ becomes an $\epsilon$-optimal arm under $\lambda$. Thus we have $\P_{\lambda}\left( B\notin \widehat{G}_\epsilon \right)\leq \delta$ while $\P_{\nu}\left( B\notin \widehat{G}_\epsilon \right) \geq 1-\delta$. Therefore 
\begin{align*}
    1- 2\delta &\leq \P_{\nu}\left( B\notin \widehat{G}_\epsilon \right) - \P_\lambda\left( B\notin \widehat{G}_\epsilon \right)\\
     &\leq \TV{\P_\nu, \P_\lambda}\\
    &\leq \TV{\P_\nu, \P_{\widetilde{\nu}}}+ \TV{\P_{\widetilde{\nu}}, \P_\lambda}\\
    &\overset{(a)}{\leq} \sum_{a\in G_{\beta_\epsilon}(\bmu)} \P_{\nu}\left( N_a(\tau) > n \right) + \sqrt{\frac{\KL(\P_{\widetilde{\nu}}, \P_\lambda)}{2}}\\
    &=\sum_{a\in G_{\beta_\epsilon}(\bmu)} \P_{\nu}\left(N_a(\tau) > n \right)  + \sqrt{\frac{n\sum_{a\in G_{\beta_\epsilon}(\bmu)} (\mu_a - \mu_1 +\beta_\epsilon +\eta)^2/2 }{2}}\\
    &\overset{(b)}{=} \sum_{a\in G_{\beta_\epsilon}(\bmu)} \P_{\nu}\left(N_a(\tau) > n \right) + \sqrt{\frac{n |G_{\beta_\epsilon}(\bmu)| (\beta_\epsilon +\eta)^2 }{4}}\\
    &\overset{(c)}{=} \sum_{a\in G_{\beta_\epsilon}(\bmu)} \P_{\nu}\left(N_a(\tau) > n \right)  + |G_{\beta_\epsilon}(\bmu)| (\beta_\epsilon +\eta)\sqrt{n}/2
\end{align*}

  where (a) comes from (\ref{eq:TV_decomposition_3}) and Pinsker's inequality, (b) is because all arms in $G_{\beta_\epsilon}(\bmu)$ satisfy $\mu_1-\beta_\epsilon \leq \mu_a\leq \mu_1$ and (c) comes from the fact that $\sqrt{|G_{\beta_\epsilon}(\bmu)|}\leq |G_{\beta_\epsilon}(\bmu)|$. Note that the inequality above holds for all $\eta > 0$. We get the final result by taking the limit $\eta \to 0$.
\end{proof}
Now we are ready to prove Theorem \ref{thm:simulator_LB}. By combining the results of Lemmas \ref{lemma:simulator_1} and \ref{lemma:simulator_2}, we get for all $b\in [K]$
\begin{align*}
    \frac{1-2\delta}{3|G_{\beta_\epsilon}(\bmu)|} -(\mu_1-\mu_b+ \beta_\epsilon)\sqrt{n}/6 \leq \P_\nu(N_b(\tau) > n)
\end{align*}
Thus by choosing $n = \floor{\frac{(1-2\delta)^2}{|G_{\beta_\epsilon}(\bmu)|^2 (\mu_1-\mu_b+ \beta_\epsilon)^2}}$ we get
\begin{align*}
   \frac{1-2\delta}{6|G_{\beta_\epsilon}(\bmu)|} \leq \P_\nu ( N_b(\tau) > n) \leq \P_\nu \big( N_b(\tau) \geq \frac{(1-2\delta)^2}{|G_{\beta_\epsilon}(\bmu)|^2 (\mu_1-\mu_b+ \beta_\epsilon)^2}\big),
\end{align*}
which implies that for all $b \in [K]$,
\begin{align*}
   \frac{(1-2\delta)^3}{6|G_{\beta_\epsilon}(\bmu)|^3 (\mu_1-\mu_b+ \beta_\epsilon)^2} \leq  \E_\nu[ N_b(\tau)].
\end{align*}
The final result is obtained by summing the inequality over all arms and noting that for $\delta \leq 10, (1-2\delta)^3\geq 1/2$.

\subsection{Example showing the improved scaling of Theorem \ref{thm:simulator_LB} w.r.t the number of arms.}
We recall that $\mu_1 = \beta, \mu_K = -\epsilon$ and $\mu_a = -\beta$ for $a\in [|2,K-1|]$. Note that in this case $\beta_\epsilon = \beta$. By symmetry, $\omega_a = \omega_2$ for all $a\in [|2,K-1|]$. In this case, using Lemma \ref{lem:bestresponse} we have
\begin{align*}
    \CT^{-1} &\overset{(a)}{=} \underset{\omega \in \Sigma_K}{\sup} \min\bigg(\frac{\omega_1 \omega_K\beta^2}{2(\omega_1+\omega_K)}, \frac{\omega_1 \omega_2(\epsilon-2\beta)^2}{2(\omega_1+\omega_2)}, \frac{\omega_2 \omega_3\epsilon^2}{2(\omega_2+\omega_3)}\bigg)\\
    &= \underset{\omega \in \Sigma_K}{\sup} \min\bigg(\frac{\omega_1 \omega_K\beta^2}{2(\omega_1+\omega_K)}, \frac{\omega_1 \omega_2(\epsilon-2\beta)^2}{2(\omega_1+\omega_2)}, \frac{\omega_2 \epsilon^2}{4}\bigg)\\
    &\geq \underset{\omega \in \Sigma_K}{\sup} \min\bigg(\frac{\omega_1 \omega_K\beta^2}{2(\omega_1+\omega_K)}, \frac{\omega_1 \omega_2(\epsilon-2\beta)^2}{2(\omega_1+\omega_2)}, \frac{\omega_2 (\epsilon-2\beta)^2}{4}\bigg).
\end{align*}
The first term of the min in (a) corresponds to the cost of making arm $K$ a good arm by simultaneously increasing its mean reward and decreasing the mean reward of the first arm, the second term to that of making arm $2$ a bad arm by simultaneously decreasing its mean reward and increasing the mean reward of the first arm. The third term, corresponds to the cost of making arm $2$ a bad arm by simultaneously decreasing its mean reward and increasing the mean reward of the arm $3$ (which is the same cost if we replace arms $2$ and $3$ by any other pair of arms in $[|2,K-1|]$). Now we look for $\bomega$ such that $\omega_1 = \omega_K > \omega_2$. This means that the third term of the min in the last line is always smaller than the second term. If we note $S$ the set of such omegas then one can write
\begin{align}
    \CT^{-1} \geq \underset{\omega \in \Sigma_K\cap S}{\sup} \min\bigg(\frac{\omega_1 \beta^2}{4}, \frac{\omega_2 (\epsilon-2\beta)^2}{4}\bigg)
\label{eq:omega_example}
\end{align}
Note that the right hand side is maximized when both terms of the min are equal. Let $\widetilde{\omega}$ be the maximizer. Then 
\begin{align*}
    2\widetilde{\omega}_1 + (K-2)\widetilde{\omega}_2 = 1, \quad \textrm{and}\ \widetilde{\omega}_1\beta^2 = \widetilde{\omega}_2(\epsilon-2\beta)^2.
\end{align*}
Solving for $\widetilde{\omega}$ and injecting in (\ref{eq:omega_example}) we get
\begin{align*}
    \CT^{-1} \geq \frac{(\epsilon-2\beta)^2 \beta^2}{8(\epsilon-2\beta)^2 + 4(K-2)\beta^2}
\end{align*}
or equivalently
\begin{align*}
    \CT \leq \frac{8}{\beta^2} +\frac{4(K-2)}{(\epsilon-2\beta)^2}.
\end{align*}
When $\beta \ll \epsilon$ and $\delta$ is fixed, this yields $\CT\log(1/\delta) = \cO(1/\beta^2 + K/\epsilon^2)$. In contrast note that for this particular instance $|G_{\beta}(\bmu)| =1$ so that the lower bound of Theorem \ref{thm:simulator_LB} is at least of order $\omega(K/\beta^2)$.

\section{Proof of Theorem 2}
We start with a few technical lemmas. The first two are adapted from \cite{GK16}:
\begin{lemma}\textsc{\big(Lemma 7, \cite{GK16}\big)}
  For all $t\geq 1$, the C-Tracking with weights $(\widetilde{\bomega}(\widehat{\bmu}_s))_{s\in \mathbb{N}^*}$ ensures that $N_a(t) \geq \sqrt{t+K^2} - 2K$ and that
  \begin{align*}
      \max\limits_{1\leq 1\leq K} \bigg|N_a(t) - \sum\limits_{s=1}^t \widetilde{\bomega}_a(\widehat{\bmu}_s) \bigg| \leq K(1+\sqrt{t})
  \end{align*}
\end{lemma} 
\begin{lemma}\textsc{\big(Lemma 19, \cite{GK16}\big)}
For $\xi > 0$, define $I_{\xi} \triangleq [\mu_1-\xi, \mu_1+\xi]\times \ldots [\mu_K-\xi, \mu_K+\xi]$. And for $T\geq 1$, consider the event: $\mathcal{E}_T = \bigcap\limits_{t = \floor{T^{1/4}}}^T \big(\widehat{\bmu}_t \in I_{\xi} \big)$. Then there exists two constants $B, C$ that only depend on $\bmu$ and $\xi$ such that
\begin{align*}
    \P\big(\mathcal{E}_T^c\big) \leq B T \exp(-C T^{1/8})
\end{align*}
where $\mathcal{E}_T^c$ is the complementary event of $\mathcal{E}_T$.
\end{lemma}
The last lemma states that $\bmu \mapsto T_{\varepsilon}(\bmu, \bomega)^{-1}$ is Lipschitz. Its proof is deferred to the end.
\begin{lemma}
  For all vectors $\bomega$ in the simplex, for all instances $\bmu$, $\bmu'$ in $[\mu_{min}, \mu_{max}]^K$ we have:
  \begin{align*}
      |T_{\varepsilon}(\bmu', \bomega)^{-1} - T_{\varepsilon}(\bmu, \bomega)^{-1} |\leq 4(\mu_{max} - \mu_{min}+\varepsilon) \norm{\bmu' - \bmu}_{\infty}.
  \end{align*}
\label{lemma:Lipschitz_2}
\end{lemma}

Now we are ready to prove the Theorem. We denote by $L_1([\mu_{min}, \mu_{max}]^K) \triangleq 4(\mu_{max} - \mu_{min}+\varepsilon)$ the Lipschitz constant of the mapping $\bmu \mapsto T_{\varepsilon}(\bmu, \bomega)^{-1}$ in the domain $[\mu_{min}, \mu_{max}]^K$ and by $L_2(\bmu) \triangleq \max_{a,b\in[K]}\frac{(\mu_a-\mu_b + \varepsilon)^2}{2}$ the Lipschitz constant of the mapping $\bomega \mapsto T_{\varepsilon}(\bmu, \bomega)^{-1}$. 

We will prove a lower bound on $T_\varepsilon\bigg(\widehat{\bmu}_t, \frac{N(t)}{t}\bigg)^{-1}$ under $\mathcal{E}_T$ which will result into an upper bound on the stopping time $\tau_\delta = \inf\big\{ t \in \mathbb{N}\ :\ t T_\varepsilon\big(\widehat{\bmu}_t, \frac{N(t)}{t} \big)^{-1} \geq \beta(\delta,t) \big\}$.
First observe that under $\mathcal{E}_T$, the $L_1$ constant is upper bounded: $L_1(I_{\xi}) \leq L_{1,max} \triangleq 4(\max_a \mu_a - \min_b \mu_b + \varepsilon + 2\xi)$. Similarly, we have for all $\floor{T^{1/4}} \leq t \leq T,\ L_2(\widehat{\bmu}_t) \leq L_{2, max} \triangleq \frac{(\max_a \mu_a - \min_b \mu_b + 2\xi + \varepsilon)^2}{2}$.
Now applying Lemma 7 and the Lipschitz property w.r.t the weights, we have for all $\floor{T^{1/4}} \leq t \leq T$:
\begin{align*}
  T_\varepsilon\bigg(\widehat{\bmu}_t, \frac{N(t)}{t}\bigg)^{-1} &\geq T_\varepsilon\bigg(\widehat{\bmu}_t, \frac{\sum\limits_{s=1}^t \widetilde{\bomega}(\widehat{\bmu}_s)}{t}\bigg)^{-1} - L_{2,max}\frac{K(1+\sqrt{t})}{t}\\
  &\geq  \frac{\sum\limits_{s=1}^t T_\varepsilon\big(\widehat{\bmu}_t, \widetilde{\bomega}(\widehat{\bmu}_s)\big)^{-1}}{t} - L_{2,max}\frac{K(1+\sqrt{t})}{t}\\
  &\geq \frac{\sum\limits_{s=\floor{T^{1/4}}}^t T_\varepsilon\big(\widehat{\bmu}_t, \widetilde{\bomega}(\widehat{\bmu}_s)\big)^{-1}}{t} - L_{2,max}\frac{K(1+\sqrt{t})}{t}\\
\end{align*}
Where we used the fact that the mapping $\bomega \mapsto  T_{\varepsilon}(\bmu, \bomega)^{-1}$ is concave (resp. non-negative) in the second (resp. last) inequality. Now observe that for all $s,t \geq \floor{T^{1/4}},\ \norm{\widehat{\bmu}_t - \widehat{\bmu}_s}_{\infty} \leq 2\xi$. Therefore the Lipschitz property w.r.t $\bmu$ implies that:
\begin{align}
  T_\varepsilon\bigg(\widehat{\bmu}_t, \frac{N(t)}{t}\bigg)^{-1} &\geq \frac{\sum\limits_{s=\floor{T^{1/4}}}^t T_\varepsilon\big(\widehat{\bmu}_s, \widetilde{\bomega}(\widehat{\bmu}_s)\big)^{-1}}{t} - \frac{2\xi L_{1,max}(t-\floor{T^{1/4}})}{t}- L_{2,max}\frac{K(1+\sqrt{t})}{t} \nonumber\\
  &\geq \frac{\sum\limits_{s=\floor{T^{1/4}}}^t T_\varepsilon^*\big(\widehat{\bmu}_s\big)^{-1}}{t}  - \frac{\sum\limits_{s=\floor{T^{1/4}}}^t \frac{1}{\sqrt{s}}}{t} -2\xi L_{1,max}- L_{2,max}\frac{K(1+\sqrt{t})}{t}
 \label{eq:111}
\end{align}
where in the second inequality we used the fact that by definition $\widetilde{\bomega}(\widehat{\bmu}_s)$ is at most $\frac{1}{\sqrt{s}}$ sub-optimal. Now observe that:
\begin{align*}
    & T_\varepsilon^*\big(\widehat{\bmu}_s\big)^{-1} \underset{s\to\infty}{\longrightarrow} T_\varepsilon^*\big(\bmu\big)^{-1}\quad \textrm{almost surely (since $N_a(t) \geq \sqrt{t+K^2} - 2K$)}.\\
    & \frac{\sum\limits_{s=\floor{T^{1/4}}}^t \frac{1}{\sqrt{s}}}{t} \underset{t\to \infty}{\sim} \frac{\int_{1}^{t} \frac{dx}{\sqrt{x}}}{t} \longrightarrow 0.\\
    & \frac{K(1+\sqrt{t})}{t} \longrightarrow 0. 
\end{align*}
Therefore for $\eta > 0$, there exists $t_\eta$ such that for all $t\geq t_\eta$:
\begin{align}
    \frac{\sum\limits_{s=\floor{T^{1/4}}}^t T_\varepsilon^*\big(\widehat{\bmu}_s\big)^{-1}}{t}  - \frac{\sum\limits_{s=\floor{T^{1/4}}}^t \frac{1}{\sqrt{s}}}{t} -L_{2,max}\frac{K(1+\sqrt{t})}{t} \geq T_\varepsilon^*\big(\bmu\big)^{-1} - \eta.
\label{eq:112}
\end{align}
Summing up (\ref{eq:111}) and (\ref{eq:112}), we get for all $t\geq t_\eta$:
\begin{align*}
   T_\varepsilon\bigg(\widehat{\bmu}_t, \frac{N(t)}{t}\bigg)^{-1} \geq T_\varepsilon^*\big(\bmu\big)^{-1} -2\xi L_{1,max} - \eta. 
\end{align*}
Therefore for every $ T \geq \max\bigg(t_\eta, \frac{\beta(\delta,t)}{T_\varepsilon^*\big(\bmu\big)^{-1} -2\xi L_{1,max} - \eta}\bigg)$, we have $\mathcal{E}_T \subset (\tau_\delta \leq T)$ thus $\P(\tau_\delta > T) \leq \P\big(\mathcal{E}_T^c\big) \leq B T \exp(-C T^{1/8})$. Hence:
\begin{align*}
    \E[\tau_\delta] &= \sum\limits_{T=1}^{\infty} \P(\tau_\delta > T)\\
    &\leq \max\bigg(t_\eta, \frac{\beta(\delta,t)}{T_\varepsilon^*\big(\bmu\big)^{-1} -2\xi L_{1,max} - \eta}\bigg) + \sum\limits_{T=1}^{\infty} B T \exp(-C T^{1/8})
\end{align*}
Since $\lim\limits_{\delta\to 0} \frac{\beta(\delta,t)}{\log(1/\delta)} = 1$ the last inequality implies that $\limsup\limits_{\delta\to 0} \frac{\E[\tau_\delta]}{\log(1/\delta)} \leq \frac{1}{T_\varepsilon^*\big(\bmu\big)^{-1} -2\xi L_{1,max} - \eta}$. We conclude by letting $\eta$ and $\xi$ go to zero.
\subsection{Proof of Lemma \ref{lemma:Lipschitz_2}}

  \underline{\textbf{First case: arms in $\bmu$ and $\bmu'$ have the same order}}\\
  
  Without loss of generality, suppose that $\mu_1 \geq \mu_2 \geq \ldots \geq \mu_K$ and $\mu'_1 \geq \mu'_2 \geq \ldots \geq \mu'_K$.
  Then we see that for all $k \neq l \in [K], \avg$ and $\avgp$ have the same formula and : $|\avgp - \avg| \leq \norm{\bmu' - \bmu}_{\infty}$, which implies that $|\blambdaklp - \blambdakl| \leq \norm{\bmu' - \bmu}_{\infty}$. Therefore, letting $f$ denote the function $f(\bmu,\blambda) \triangleq \suma \omega_a \frac{(\mu_a - \lambda_a)^2}{2}$, we have:
 
\begin{align*}
    |f(\bmu',\blambdaklp) - f(\bmu,\blambdakl)| &\leq \suma \frac{\omega_a(\mu'_a - \mu_a + \blambdakl_a - \blambdaklp_a)(\mu'_a + \mu_a - \blambdakl_a - \blambdaklp_a)}{2}\\
    &\leq \frac{\omega_a \times 2\norm{\bmu' - \bmu}_{\infty} \times 2(\mu_{max} - \mu_{min}+\varepsilon)}{2}\\
    & = 2(\mu_{max} - \mu_{min}+\varepsilon)\norm{\bmu' - \bmu}_{\infty}.
\end{align*}
where in the second inequality we used the fact that $\blambdakl$ (resp. $\blambdaklp$) is a weighted average of some arm in $\bmu$(resp. $\bmu'$) with one or more arms of $\bmu$(resp. $\bmu'$) decreased by $\varepsilon$ and therefore lies in $[\mu_{min}-\varepsilon, \mu_{max}]^K$. Let $(k_0, l_0)$ be such that $\blambdastar = \blambda^{k_0,l_0}_{\varepsilon}(\bomega)$ then:
\begin{align}
    T_{\varepsilon}(\bmu', \bomega)^{-1} - T_{\varepsilon}(\bmu, \bomega)^{-1} &= T_{\varepsilon}(\bmu', \bomega)^{-1} - f\big(\omega, \blambda^{k_0, l_0}_{\varepsilon}(\bomega)\big) \\
    &\leq f\big(\omega,\blambda'^{k_0, l_0}_{\varepsilon}(\bomega)\big) - f\big(\omega, \blambda^{k_0, l_0}_{\varepsilon}(\bomega)\big) \\
    &\leq 2(\mu_{max} - \mu_{min}+\varepsilon)\norm{\bmu' - \bmu}_{\infty}.
\end{align}
By symmetry we get for all instances $\bmu$ and $\bmu'$ with the same arm ordering: 
\begin{align*}
    |T_{\varepsilon}(\bmu', \bomega)^{-1} - T_{\varepsilon}(\bmu, \bomega)^{-1}| \leq 2(\mu_{max} - \mu_{min}+\varepsilon)\norm{\bmu' - \bmu}_{\infty}.
\end{align*}
\underline{\textbf{Second case: arms in $\bmu$ and $\bmu'$ have a different order}}\\
Then for $n$ large enough we can find a sequence $(\mu^i)_{0\leq i \leq  2^n}$ of instances in the segment $[\bmu, \bmu']$ such that $\mu^0 = \bmu$, $\mu^{2^n} = \bmu'$ and:
\begin{equation*}
    \forall i \in [|0, 2^n-1|],\ \mu^i \textrm{ and } \mu^{i+1} \textrm{ have the same arm ordering and } \norm{\mu^{i+1} - \mu^{i}}_{\infty} \leq \frac{\norm{\bmu' - \bmu}_{\infty}}{2^{n-1}}.
\end{equation*}
We can construct such a sequence in the following way: Split $[\mu_{min}, \mu_{max}]^K$ into $K!$ regions such that any two instances in the same region share the same arm ordering. The boundaries between these regions correspond to instances where two or more arms are equal. Starting from $\mu^0 \triangleq \bmu$, span the segment $[\bmu, \bmu']$ and define $\mu^{i+1}$ to be the first instance where: either the $L^\infty$ distance from $\mu^i$ is equal to $\frac{\norm{\bmu' - \bmu}_{\infty}}{2^{n-1}}$, or we cross a boundary between two regions. Since there can be at most $K! - 1$ changes in the arm ordering, for $n$ large enough such sequence always exists. Now we have:
\begin{align*}
    |T_{\varepsilon}(\bmu', \bomega)^{-1} - T_{\varepsilon}(\bmu, \bomega)^{-1}| &\leq \sum\limits_{i=0}^{2^n - 1} |T_{\varepsilon}(\bomega, \bmu^{i+1})^{-1} - T_{\varepsilon}(\bomega, \bmu^i)^{-1}| \\ 
    &\leq \sum\limits_{i=0}^{2^n - 1} 2(\mu_{max} - \mu_{min}+\varepsilon) \frac{\norm{\bmu' - \bmu}_{\infty}}{2^{n-1}}\\
    &\leq 4(\mu_{max} - \mu_{min}+\varepsilon)\norm{\bmu' - \bmu}_{\infty}. 
\end{align*}
where in the second inequality we use the first case and the fact that $\mu^i$  and $\mu^{i+1}$  have the same arm ordering. As a summary, we always have:
  \begin{align*}
      |T_{\varepsilon}(\bmu', \bomega)^{-1} - T_{\varepsilon}(\bmu, \bomega)^{-1} |\leq 4(\mu_{max} - \mu_{min}+\varepsilon) \norm{\bmu' - \bmu}_{\infty}.
  \end{align*}

\end{document}